\documentclass[11pt]{article}

\PassOptionsToPackage{table,xcdraw}{xcolor}
\usepackage[preprint]{acl}

\usepackage{times}
\usepackage{latexsym}
\usepackage[T1]{fontenc}
\usepackage[utf8]{inputenc}
\usepackage{microtype}
\usepackage{inconsolata}
\usepackage{graphicx}

\usepackage{bm}
\usepackage{amsfonts}
\usepackage{amsmath}
\usepackage{textcomp}
\usepackage{placeins}
\usepackage{array}
\usepackage{pifont}
\usepackage{makecell}
\usepackage{booktabs}
\usepackage{amsthm}
\usepackage{multirow}
\usepackage{subcaption}
\usepackage{pgf}
\usepackage{pgffor}
\usepackage{etoolbox}
\usepackage[normalem]{ulem}
\usepackage[most]{tcolorbox}
\tcbuselibrary{breakable}
\usepackage{tabularx}
\usepackage[linesnumbered,ruled,vlined]{algorithm2e}
\usepackage{soul}
\usepackage{float}
\usepackage{enumitem}

\newtheorem{example}{Example}

\newtheorem{theorem}{Theorem}

\usepackage{newunicodechar}
\newunicodechar{–}{--}
\definecolor{revisionblue}{RGB}{0,0,255}

\newcommand{\ModelName}{MSR-MEL}

\title{Multi-Perspective Evidence Synthesis and Reasoning for Unsupervised Multimodal Entity Linking}

\author{
    \textbf{Mo Zhou}\textsuperscript{1} \quad
    \textbf{Jianwei Wang}\textsuperscript{1,*} \quad
    \textbf{Kai Wang}\textsuperscript{1} \\
    \textbf{Helen Paik}\textsuperscript{1} \quad
    \textbf{Ying Zhang}\textsuperscript{2} \quad
    \textbf{Wenjie Zhang}\textsuperscript{1} \\[0.8em]
    \textsuperscript{1}The University of New South Wales, Sydney, Australia \\
    \textsuperscript{2}University of Technology Sydney, Sydney, Australia \\[0.5em]
    \texttt{mo.zhou4@unsw.edu.au} \quad
    \texttt{jianwei.wang1@unsw.edu.au} \\
    \texttt{w.kai@sjtu.edu.cn} \quad
    \texttt{h.paik@unsw.edu.au} \\
    \texttt{Ying.Zhang@uts.edu.au} \quad
    \texttt{wenjie.zhang@unsw.edu.au} \\[0.5em]
    \textsuperscript{*}Corresponding author
}

\begin{document}
\maketitle

\begin{abstract}
    Multimodal Entity Linking (MEL) maps ambiguous multimodal mentions to entities in a knowledge base.
    However, existing methods mainly focus on instance-centric representations, leaving broader forms of evidence and their interdependencies insufficiently explored. Motivated by the observation that human experts make decisions by jointly considering evidence from multiple perspectives,
    we propose \textbf{\ModelName}, a \textbf{\ul{M}}ulti-perspective Evidence \textbf{\ul{S}}ynthesis and \ul{\textbf{R}}easoning framework with Large Language Models (LLMs) for unsupervised \ul{\textbf{MEL}}.
    Specifically, we adopt a two-stage framework:
    \textit{(1) Offline Multi-Perspective Evidence Synthesis} constructs instance-centric, group-level, lexical, and statistical evidence. 
    In particular, group-level evidence is learned from LLM-enhanced contextualized graphs through asymmetric teacher--student alignment.
    \textit{(2) Online Multi-Perspective Evidence Reasoning} leverages heterogeneous evidence to induce an LLM-driven decision tree for candidate reranking.
    Extensive experiments on widely used MEL benchmarks demonstrate that \ModelName{} consistently outperforms state-of-the-art unsupervised methods.
    The source code of this paper is available at: \url{https://anonymous.4open.science/r/MSR-MEL-C21E/}.
    \end{abstract}

\section{Introduction}
\label{sec:intro}
Entity Linking (EL) is a core task in intelligent information processing that maps ambiguous mentions to specific entries in a Knowledge Base (KB)\cite{cucerzan2007large,delpeuch2019opentapioca,dubey2018earl,fang2019joint,chen2024knowledge}. With the rapid growth of multimedia data on the web, this task has extended to Multimodal Entity Linking (MEL). Unlike traditional EL which relies solely on text, MEL commonly utilizes textual context and visual information to identify entities~\cite{moon2018multimodal}. As a pivotal upstream task, MEL plays a vital role in various real-world applications, such as multimodal information retrieval, question answering, and social media recommendation~\cite{de2019question, longpre2021entity, xiong2019improving, geng2022recommendation}.

\begin{example}
Figure~\ref{fig:example} illustrates an ambiguous mention of ``Oxford'' with textual and visual context.
Among several name-matched candidates, only the University of Oxford is the correct entity.
\end{example}

\begin{figure}[t!]
    \centering
    \includegraphics[width=1\linewidth]{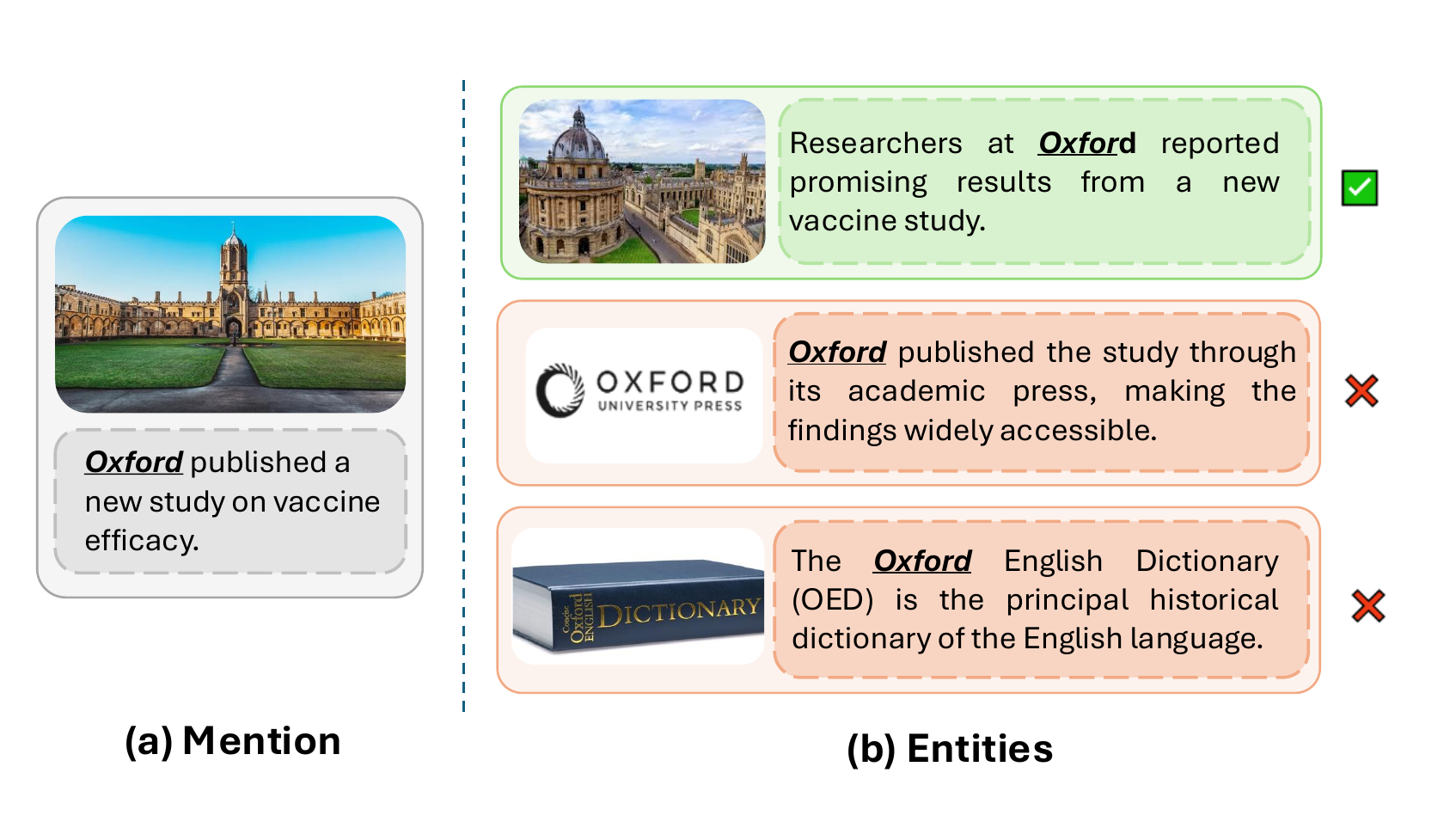}
    \caption{An example of multimodal entity linking.}
    \label{fig:example}
\end{figure}
Given its significance, MEL has attracted increasing research attention. Existing approaches are typically categorized into supervised and unsupervised paradigms, with representative methods compared in Appendix~\ref{app:method_comparison} (Table~\ref{tab:comparison}).
Supervised methods model the MEL task as a classification problem on annotated mention–entity pairs, learning a score that indicates the confidence in assigning multimodal mentions to entities~\cite{eshel2017named, wang2022multimodal, song2024dual, hu2025multi, liu2024unimel, moon2018multimodal, luo2023multi}. 
For example, DWE enhances multimodal matching through query and entity refinement~\cite{song2024dual}, while M$^3$EL models finer-grained intra- and cross-modal interactions~\cite{hu2025multi}. UniMEL and KGMEL further incorporate Large Language Models (LLMs) for contextual enhancement or candidate selection~\cite{liu2024unimel,kim2025kgmel}.

Supervised methods rely heavily on large amounts of labeled data in MEL, and collecting such multimodal mention–entity annotations is expensive.
To alleviate the dependence on manual annotation, recent unsupervised methods have emerged. These approaches generally leverage multimodal models to generate representations and perform 
MEL via ranking on the similarity score of the representations~\cite{luo2023multi, shi2024generative, zhu2025openmel}.
Among these unsupervised approaches, BERT-based methods serve as simple baselines~\cite{devlin2019bert}, while OpenMEL and DeepMEL further introduce LLMs-based enhancement or agent-style interaction~\cite{zhu2025openmel,wang2026deepmel}.

\noindent \textbf{Motivations.}
Despite recent progress, current supervised and unsupervised MEL methods primarily focus on optimizing instance-centric multimodal features through various representation learning techniques~\cite{wang2024efficient,xing2023drin}.
Although some methods incorporate contextual information through contrastive learning, they do not directly fuse neighborhood features~\cite{luo2024bridging}.
As a result, they fail to exploit broader forms of evidence and capture the intricate relationships within multimodal signals, limiting their effectiveness in complex MEL settings where data are often incomplete or inconsistent~\cite{wu2024deep,zhang2024multimodal}.
As shown in Table~\ref{tab:statistics_of_datasets}, the entity image ratios of WikiMEL, RichpediaMEL, and WikiDiverse are only 67.26\%, 57.41\%, and 50.16\%, respectively.

These observations suggest that relying solely on instance-centric evidence is insufficient for robust MEL.
In contrast, human expert decision-making rarely relies on a single cue. Instead, humans naturally aggregate evidence from multiple perspectives~\cite{de2010accumulation}, such as semantic relevance, contextual coherence, and relational consistency, and reason over their consistency before reaching a decision.
This motivates the use of multi-perspective evidence with explicit reasoning over their interactions for MEL.

\noindent \textbf{Challenges.}
Unsupervised MEL with multi-perspective evidence presents two key challenges:

\noindent \textit{Challenge I: How to extract and synthesize stable, complementary multi-perspective evidence from sparse and low-quality multimodal data?}
Real-world multimodal data often contain missing modalities and inconsistent-quality signals~\cite{xu2021unsupervised}.
Moreover, textual and visual cues may capture information at different granularities and are not always semantically aligned~\cite{li2024multimodal}, making robust evidence synthesis difficult.

\noindent \textit{Challenge II: How to effectively reason over multi-perspective evidence in a fully unsupervised setting?}
Different evidence types may provide complementary or conflicting signals. Without labeled data, determining their reliability and relative importance for entity linking remains difficult.

\noindent \textbf{Ours approaches.}
To address these challenges, we propose
\textbf{\ModelName}, a \textbf{\ul{M}}ulti-perspective Evidence \textbf{\ul{S}}ynthesis and \ul{\textbf{R}}easoning framework with LLMs for unsupervised \ul{\textbf{MEL}}.
Here, unsupervised denotes label-free MEL training,
with self-supervised graph representation learning.
Specifically, MSR-MEL follows a two-stage design:

\noindent \ul{\textit{(1) Offline Stage: Multi-Perspective Evidence Synthesis.}}
To address \textit{Challenge I}, the offline stage synthesizes four complementary evidence types: instance-centric evidence from multimodal representations, group-level evidence from neighborhood context, lexical evidence from string matching, and statistical evidence summarizing evidence strength and reliability.
To obtain robust group-level evidence, we construct LLM-enhanced contextualized graphs and learn modality-specific representations through asymmetric teacher--student graph alignment.
This process aggregates relational context while reducing the influence of sparse or noisy multimodal evidence.

\noindent \ul{\textit{(2) Online Stage: Multi-Perspective Evidence Reasoning.}}
To address \textit{Challenge II}, the online stage first retrieves a high-recall candidate set and then uses an LLM-induced decision tree to rerank the candidates.
The tree evaluates the synthesized evidence along an interpretable reasoning path for each mention--entity pair.

\noindent\textbf{Contributions.}

\begin{itemize}[
    leftmargin=1.4em,
    labelsep=0.5em,
    topsep=2pt,
    itemsep=2pt,
    parsep=0pt,
    partopsep=0pt
]
    \item We propose \textbf{\ModelName}, a two-stage framework for unsupervised MEL with offline multi-perspective evidence synthesis and online LLM-driven evidence reasoning.
    \item We introduce multi-perspective evidence synthesis and reasoning that jointly integrate instance-centric, group-level, lexical, and statistical evidence, with query-specific LLM-induced decision trees for candidate re-ranking.
    \item Experiments on three MEL benchmarks show consistent gains over unsupervised baselines, including up to 9.89 Hit@1 points over OpenMEL~\cite{zhu2025openmel}, with improved robustness and efficiency.
\end{itemize}

\section{Preliminaries}
\label{sec:preliminaries}
We formally define the multimodal entity linking (MEL) task.
Let $\mathcal{M}=\{m_i\}_{i=1}^{|\mathcal{M}|}$ and
$\mathcal{E}=\{e_j\}_{j=1}^{|\mathcal{E}|}$ denote the mention and entity
sets, respectively.
A mention $m_i=(n_i,t_i,v_i)$ consists of a surface form, textual context,
and image, while an entity $e_j=(n_j,t_j,v_j)$ consists of a canonical name,
textual description, and image.
The visual component may be missing or noisy.
Given a candidate set $\mathcal{C}_i\subseteq\mathcal{E}$, MEL predicts
\begin{equation}
    \hat{e}_i
    =
    \arg\max_{e_j\in\mathcal{C}_i}
    \Phi(m_i,e_j),
\end{equation}
where $\Phi(m_i,e_j)$ measures the compatibility between mention $m_i$ and entity $e_j$.

\section{Multi-Perspective Evidence Synthesis}
\label{sec:offline}
\begin{figure*}
    \centering
    \includegraphics[width=1\linewidth]{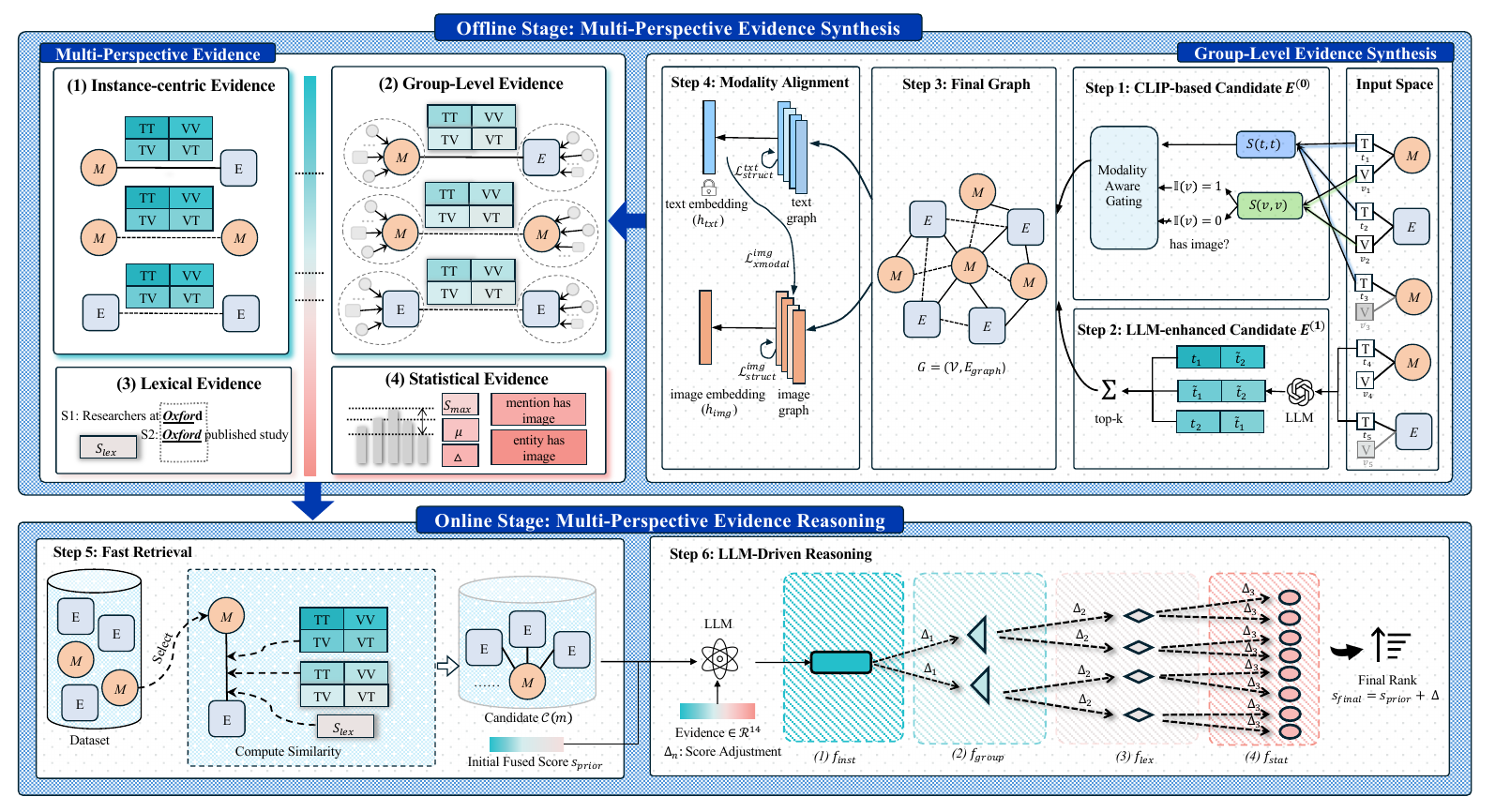}
    \caption{Overview of the proposed \ModelName{} framework. The offline stage
    synthesizes four evidence types for each mention--entity pair:
    instance-centric, group-level, lexical, and statistical evidence.
    Group-level evidence is obtained by constructing contextualized graphs
    with CLIP-based gating and LLM-enhanced semantic expansion, followed by
    asymmetric teacher--student alignment. In the online stage, multi-view
    retrieval constructs a high-recall candidate set and computes an initial
    prior score. An LLM-induced decision tree then reasons over the
    multi-perspective evidence and reranks the candidates to produce the
    final prediction.}
    \label{fig:framework}
\end{figure*}

The overall architecture of \ModelName{} is illustrated in Figure~\ref{fig:framework}.
\subsection{Multi-Perspective Evidence Formulation}
For each mention--entity pair, we construct a 14-dimensional evidence vector that integrates instance-level, group-level, lexical, and reliability evidence:
\begin{equation}
\label{eq:feat_all}
\small
\begin{aligned}
\mathbf{f}(m,e)
&=
\big[
\mathbf{f}_{\text{inst}}(m,e),\
\mathbf{f}_{\text{group}}(m,e),\\
&\quad
\mathbf{f}_{\text{lex}}(m,e),\
\mathbf{f}_{\text{stat}}(m,e)
\big]
\in \mathbb{R}^{14}.
\end{aligned}
\end{equation}
To theoretically motivate the multi-perspective evidence design, we provide an idealized risk analysis in Appendix~\ref{app:theoretical_analysis}.

\noindent \textbf{(a) Instance-centric Evidence ($\mathbf{f}_{\text{inst}}(m,e)$).}
Using CLIP embeddings~\cite{radford2021learning}, we compute mention--entity cosine similarities over four modality combinations: text--text (TT), text--image (TV), image--text (VT), and image--image (VV). The instance-centric evidence is:
\begin{equation}
\label{eq:feat_orig}
\mathbf{f}_{\text{inst}}(m,e)
=
\big[
s^{\text{TT}}_{\text{inst}},
s^{\text{TV}}_{\text{inst}},
s^{\text{VT}}_{\text{inst}},
s^{\text{VV}}_{\text{inst}}
\big].
\end{equation}

\noindent \textbf{(b) Group-level Evidence ($\mathbf{f}_{\text{group}}(m,e)$).}
We similarly compute group-level similarities from graph-enhanced representations:
\begin{equation}
\label{eq:feat_gnn}
\mathbf{f}_{\text{group}}(m,e)
=
\big[
s^{\text{TT}}_{\text{group}},
s^{\text{TV}}_{\text{group}},
s^{\text{VT}}_{\text{group}},
s^{\text{VV}}_{\text{group}}
\big].
\end{equation}

\noindent \textbf{(c) Lexical Evidence ($\mathbf{f}_{\text{lex}}(m,e)$).}
We use a normalized, case-insensitive string matching score
$s_{\mathrm{lex}}(m,e)\in[0,1]$ between the mention surface form and
entity name:
\begin{equation}
\label{eq:feat_str}
\mathbf{f}_{\mathrm{lex}}(m,e)
=
\left[s_{\mathrm{lex}}(m,e)\right].
\end{equation}
The matching algorithm is detailed in Appendix~\ref{app:lexical_evidence}.

\noindent \textbf{(d) Statistical Evidence ($\mathbf{f}_{\text{stat}}(m,e)$).}
To characterize evidence reliability, we summarize the preceding similarity scores using their mean $\mu(m,e)$, maximum $s_{\max}(m,e)$, and dominance gap $\Delta(m,e)$, 
specifically:

\begin{equation}
    \label{eq:mu}
    \small
    \mu(m,e)
    =
    \frac{
    \sum s,\quad
    s \in
    \left\{
    \mathbf{f}_{\text{inst}}(m,e),
    \mathbf{f}_{\text{group}}(m,e),
    \mathbf{f}_{\text{lex}}(m,e)
    \right\}
    }{
    \left|
    \mathbf{f}_{\text{inst}}(m,e)
    \right|
    +
    \left|
    \mathbf{f}_{\text{group}}(m,e)
    \right|
    +
    \left|
    \mathbf{f}_{\text{lex}}(m,e)
    \right|
    }.
\end{equation}

\begin{equation}
\Delta(m,e)=s_{\max}(m,e)-\mu(m,e).
\end{equation}

A large gap accompanied by a low mean typically indicates over-reliance on a single weak cue.

We further encode modality availability using binary indicators that denote whether visual information is present for the mention and the entity.
Accordingly, the statistical evidence is defined as:

\begin{equation}
\label{eq:f_stat}
\small
\mathbf{f}_{\text{stat}}(m,e)
=
\left[
\begin{aligned}
&\mu(m,e),\;
s_{\max}(m,e),\;
\Delta(m,e),\\
&\mathbb{I}_{\text{img}}(m),\;
\mathbb{I}_{\text{img}}(e)
\end{aligned}
\right].
\end{equation}

\subsection{Contextualized Graph Construction}
\label{subsec: graph_construction}
To obtain group-level evidence, we construct an undirected multimodal graph $\mathcal{G}=(\mathcal{V}, E_{graph})$ where $\mathcal{V}=\mathcal{M}\cup\mathcal{E}$. Each node $v_i$ is initialized with a CLIP text embedding $\mathbf{t}_{v_i}$ and, when available, a visual embedding $\mathbf{v}_{v_i}$.

\noindent \textbf{Step 1: CLIP-based Candidate Selection.}
For a node pair $(v_i,v_j)$, we define the modality-aware similarity to construct coarse candidate edges:
\begin{equation}
\label{eq:gated_sim}
s_{\text{gate}}(v_i, v_j)
= \cos(\mathbf{t}_{v_i}, \mathbf{t}_{v_j})
+ \mathbb{I}_{ij} \cdot \cos(\mathbf{v}_{v_i}, \mathbf{v}_{v_j}),
\end{equation}
where $\mathbb{I}_{ij}=1$ when both images are available and $0$
otherwise. The initial edge set is filtered by:
\begin{equation}
    E^{(0)}
    =
    \left\{
    (v_i, v_j) \in \mathcal{V} \times \mathcal{V}
    \;\middle|\;
    s_{\text{gate}}(v_i, v_j) \ge \delta_{\text{gate}},
    \right\}.
\end{equation}
where $\delta_{\text{gate}}$ is a predefined similarity threshold.

\noindent \textbf{Step 2: LLM-enhanced Candidate Selection.}
To resolve ambiguity, we use LLMs to generate contextualized semantic descriptions.
Specifically, we use an LLM to generate enhanced textual descriptions for all mentions $m \in \mathcal{M}$ and for a selected subset of entities $\mathcal{E}_k \subset \mathcal{E}$, encoded as:

\begin{equation}
\label{eq:llm_clip_aug}
\small
    \tilde{\mathbf{t}}_x
    =
    \mathrm{CLIP}_{\text{text}}
    \!\left(
    \mathrm{LLM}_{\theta}\big(\mathrm{Prompt}(x)\big)
    \right),
    \quad
    x \in \mathcal{M} \cup \mathcal{E}_k .
\end{equation}

We combine original and enhanced textual representations using
\begin{equation}
\label{eq:llm_fusion}
\begin{aligned}
s_{\mathrm{llm}}(v_i,v_j)
={}&
w_1\cos(\mathbf{t}_{v_i},\tilde{\mathbf{t}}_{v_j})
+
w_2\cos(\tilde{\mathbf{t}}_{v_i},\mathbf{t}_{v_j})\\
&+
w_3\cos(\tilde{\mathbf{t}}_{v_i},\tilde{\mathbf{t}}_{v_j}).
\end{aligned}
\end{equation}
For each node, its Top-$K_{\mathrm{llm}}$ neighbors under
$s_{\mathrm{llm}}$ form the semantic edge set $E^{(1)}$.
The final graph is
\begin{equation}
\label{eq:E_union}
E_{\mathrm{graph}}=E^{(0)}\cup E^{(1)}.
\end{equation}

\subsection{Asymmetric Teacher–Student Alignment}
\label{subsec: graph_learning}
Textual information is generally more complete than visual information
in MEL benchmarks. We therefore first train a text-based GNN as a
teacher and then use it to guide an image-based student.
Appendix~\ref{app:theoretical_analysis} provides an error-bound analysis
supporting the asymmetric teacher--student alignment.

\noindent \textbf{4.3.1 Shared Subgraph Encoding.} 
For each node $v_i$, we extract a local PPR subgraph
$\mathcal{G}_{v_i}$~\cite{gasteiger2018predict}.
The text and image views share the same topology but use text and image
features, respectively; missing image features are zero-masked.
Using GCN encoders~\cite{kipf2016semi}, for
$q\in\{\mathrm{txt},\mathrm{img}\}$ we compute:
\begin{equation}
\label{eq:graph_encoding_compact}
\begin{aligned}
\mathbf{h}^{q}_{v_i}
&=
f_{\theta_q}
\left(
\mathcal{G}_{v_i},
\{\mathbf{x}^{q}_v\}_{v\in\mathcal{V}_{v_i}}
\right),\\
\mathbf{s}^{q}_{v_i}
&=
\frac{1}{|\mathcal{V}_{v_i}|}
\sum_{v\in\mathcal{V}_{v_i}}
\mathbf{h}^{q}_{v}.
\end{aligned}
\end{equation}

\noindent\textbf{4.3.2 Text-based Teacher GNN Training.}
We first train a text-based GNN using CLIP text embeddings
$\mathbf{x}_v=\mathbf{t}_v$ to learn stable contextualized
representations from the constructed graph.
Although the LLM outputs are not directly used as supervision, the
graph topology contains LLM-enhanced semantic edges introduced during
graph construction.
Thus, the LLM contributes implicitly through the neighborhood structure.

We adopt a structural contrastive objective to align each target-node
representation with its local subgraph context.
For a target node $v_i$, we treat
$(\mathbf{h}^{\mathrm{txt}}_{v_i},
\mathbf{s}^{\mathrm{txt}}_{v_i})$
as a positive pair and construct a negative pair by matching
$\mathbf{h}^{\mathrm{txt}}_{v_i}$ with the summary
$\mathbf{s}^{\mathrm{txt}}_{v_j}$ of another node
$v_j\neq v_i$ sampled from the same mini-batch~\cite{jiao2020sub}.
The structural contrastive loss is defined as:

\begin{equation}
\label{eq:struct_loss}
\footnotesize
\resizebox{\columnwidth}{!}{$\displaystyle
\mathcal{L}_{\mathrm{struct}}^{\mathrm{txt}}
=
\frac{1}{|\mathcal{B}|}
\sum_{v_i\in\mathcal{B}}
\max\!\left(
0,\,
\eta
-
\sigma\!\left(
{\mathbf{h}^{\mathrm{txt}}_{v_i}}^{\top}
\mathbf{s}^{\mathrm{txt}}_{v_i}
\right)
+
\sigma\!\left(
{\mathbf{h}^{\mathrm{txt}}_{v_i}}^{\top}
\mathbf{s}^{\mathrm{txt}}_{v_j}
\right)
\right)
$}
\end{equation}

where $\mathcal{B}$ is a mini-batch, $\eta$ is the margin, and
$\sigma(\cdot)$ denotes the sigmoid function.
The resulting text GNN serves as the teacher for subsequent cross-modal
distillation, with
$\mathcal{L}_{\mathrm{text}}
=\mathcal{L}_{\mathrm{struct}}^{\mathrm{txt}}$.

\noindent \textbf{4.3.3 Image-based Student GNN Distillation.} 
Because visual features are often noisy or missing, we freeze the
text-based GNN as a teacher and use it to guide visual representation learning. Given the same subgraph $\mathcal{G}_{v_i}$, the teacher produces a fixed representation $\mathbf{h}^{\text{txt}}_{v_i}$, while the image-based student takes CLIP image features $\mathbf{x}_v=\mathbf{v}_v$ as input and produces $\mathbf{h}^{\text{img}}_{v_i}$.

We align the student with the teacher using a cross-modal contrastive
objective. For each node $v_i$ in a mini-batch $\mathcal{B}$,
$(\mathbf{h}^{\text{img}}_{v_i},
\mathbf{h}^{\text{txt}}_{v_i})$ forms a positive pair, while the teacher
representations of other nodes in the batch serve as negatives.
The corresponding InfoNCE loss is~\cite{oord2018representation}:

\begin{equation}
\label{eq:xmodal_loss}
\footnotesize
    \mathcal{L}_{\text{xmodal}}
    =
    -\frac{1}{|\mathcal{B}|}
    \sum_{v_i \in \mathcal{B}}
    \log
    \frac{
    \exp\!\left(
    \mathrm{sim}(\mathbf{h}^{\text{img}}_{v_i}, \mathbf{h}^{\text{txt}}_{v_i}) / \tau
    \right)
    }{
    \sum_{v_j \in \mathcal{B}}
    \exp\!\left(
    \mathrm{sim}(\mathbf{h}^{\text{img}}_{v_i}, \mathbf{h}^{\text{txt}}_{v_j}) / \tau
    \right)
    },
\end{equation}
where $\mathrm{sim}(\cdot,\cdot)$ denotes cosine similarity and $\tau$
is the temperature parameter.

We additionally apply the structural contrastive objective in
Eq.~\ref{eq:struct_loss} to the image view, yielding
$\mathcal{L}_{\text{struct}}^{\text{img}}$.
The overall student objective is
\begin{equation}
\label{eq:img_loss}
\mathcal{L}_{\text{img}}
=
\mathcal{L}_{\text{struct}}^{\text{img}}
+
\lambda_1\mathcal{L}_{\text{xmodal}},
\end{equation}
where $\lambda_1$ balances structural learning and cross-modal
distillation.
The trained teacher and student provide complementary graph-enhanced
representations for subsequent evidence reasoning.

\section{Multi-Perspective Evidence Reasoning}
\label{sec:online}
Following offline multi-perspective evidence synthesis, we introduce the
online inference stage of \ModelName{}.
Although LLMs exhibit strong reasoning capabilities, directly applying
them to end-to-end entity linking over a large candidate space is
computationally expensive.
We therefore adopt an efficient two-stage procedure:
(1) candidate selection (Section~\ref{subsec:candidate_selection}),
which retrieves a compact, high-recall candidate set; and
(2) LLM-driven evidence re-ranking
(Section~\ref{subsec:llm_reasoning}),
which induces an interpretable decision tree to refine candidate scores.

\subsection{Candidate Selection}
\label{subsec:candidate_selection}

The online stage first constructs a high-recall candidate set through multi-perspective retrieval. We combine instance-centric, group-level, and lexical evidence to retrieve complementary candidates from multiple views.

Let $\mathcal{S}=\{\mathrm{TT},\mathrm{TV},\mathrm{VT},\mathrm{VV}\}$ denote the four modality combinations and $\mathcal{T}=\{\mathrm{inst},\mathrm{group}\}$ the two representation types. These combinations form eight dense retrieval channels, supplemented by one lexical channel. The candidate set is defined as:
\begin{equation}
\small
    \mathcal{C}(m) = \left( \bigcup_{t \in \mathcal{T}} \bigcup_{s \in \mathcal{S}} \text{Top-}K_{\text{ch}}(s_{t}^{s}(m,e)) \right) \cup \text{Top-}K_{\text{ch}}(s_{lex}(m,e))
\end{equation}
where $s_t^s$ denotes the cosine similarity under view $s$ and representation type $t$, and $s_{\mathrm{lex}}$ denotes lexical similarity. Dense candidates are retrieved using FAISS, and the results from all nine channels are merged by set union. We set $K_{\mathrm{ch}}=250$ for each channel, resulting in at most $9K_{\mathrm{ch}}=2250$ retrieved entries and no more than 2250 unique candidates after merging.

For each candidate $e\in\mathcal{C}(m)$, we compute an initial prior score $s_{\mathrm{prior}}(m,e)$ by aggregating its normalized scores across the retrieval channels. This score provides the initial ranking for subsequent evidence reasoning.

\subsection{LLM-Driven Re-ranking}
\label{subsec:llm_reasoning}

Candidate selection produces a high-recall candidate set
$\mathcal{C}(m)$, but the prior score $s_{\mathrm{prior}}(m,e)$ is based on numerical aggregation and may be dominated by misleading evidence.
For each query, the LLM induces one executable decision tree shared by all its candidates. The tree structure, thresholds, and score adjustments are generated under fixed reasoning guidelines (Prompt Box~\ref{box:prompt_design_2}), without labels, candidate identities, or feedback.

\noindent \textbf{LLM-Guided Evidence Reasoning.}
As detailed in \hyperref[box:prompt_design_2]{Prompt Box~B}, the LLM is provided with the prior score, the multi-perspective evidence vector, and reasoning guidelines that encourage cross-view consistency, incorporate modality-availability indicators, and discourage over-reliance on any single evidence.
It then generates a decision tree whose decision nodes evaluate specific evidence features, such as lexical, instance-level, group-level, and statistical evidence. Each traversed branch contributes an LLM-specified
score adjustment $\Delta_n$.
\begin{equation}
\label{eq:llm_reranking}
    s_{\mathrm{LLM}}(m,e)
    =
    s_{\mathrm{prior}}(m,e)
    +
    \sum_{n\in\mathcal{P}(m,e)}
    \Delta_n,
\end{equation}
where $\mathcal{P}(m,e)$ denotes the decision path followed by the candidate. Candidates are ranked in descending order of $s_{\mathrm{LLM}}(m,e)$.

\begin{figure}[!t]
    \centering
    \includegraphics[width=1.05\linewidth]{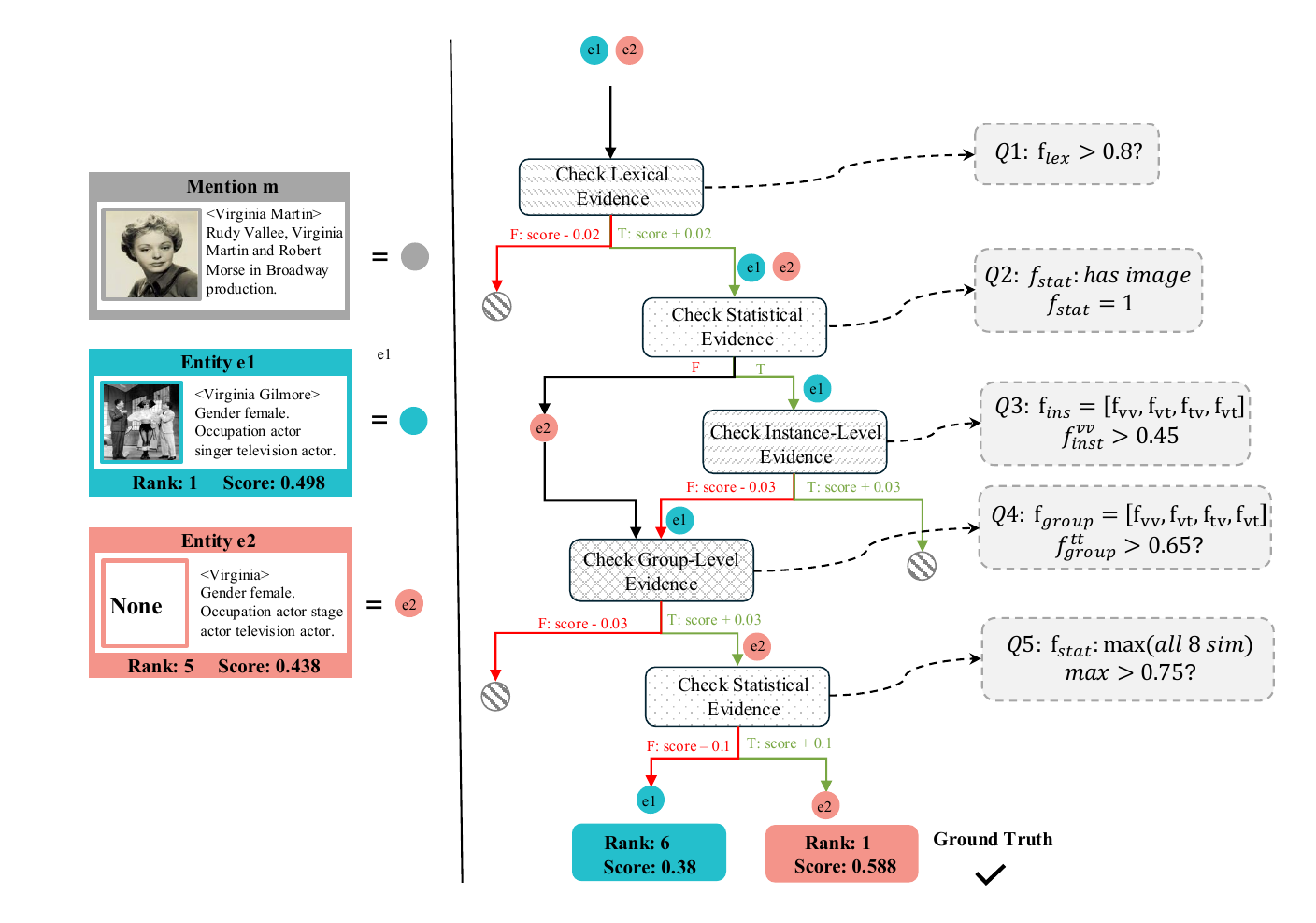}
    \caption{Example of query-specific LLM-driven evidence re-ranking. The induced decision tree (Q1–Q5) corrects a visually biased prior ranking and moves the ground-truth entity \(e_2\) to rank 1.}
    \label{fig:case_study}
\end{figure}

Figure~\ref{fig:case_study} illustrates the complete decision process. The tree first examines lexical evidence (Q1) and then checks image availability (Q2) to determine whether visual evidence should be considered.
For candidates with available images, it evaluates instance-level visual similarity (Q3) and subsequently checks whether this evidence is supported by group-level contextual evidence (Q4). Finally, it assesses the overall strength of the similarity evidence using statistical features (Q5).
Each branch applies a positive or negative score adjustment to the prior score. 
In this example, the strong visual similarity of $e_1$ receives insufficient support from the remaining evidence and is therefore penalized, whereas the more consistent evidence for $e_2$ promotes it from rank 5 to rank 1.

\section{Experiments}
\label{sec:experiments}
\subsection{Experimental Setup}

\begin{table}[t!]
    \centering
    \caption{The Statistics of MEL Datasets}
    \label{tab:statistics_of_datasets}

    \resizebox{\columnwidth}{!}{%
        \begin{tabular}{|l|*{3}{>{\centering\arraybackslash}m{3.2cm}|}}
            \hline

            \multicolumn{1}{|c|}{\textbf{Statistics}} &
            \makecell[c]{WikiMEL\\\cite{wang2022multimodal}} &
            \makecell[c]{RichpediaMEL\\\cite{wang2022multimodal}} &
            \makecell[c]{WikiDiverse\\\cite{wang2022wikidiverse}} \\
            
            \hline

            \# mentions &
            25,846 &
            17,805 &
            15,093 \\

            \# images of mentions &
            22,136 (85.6\%) &
            15,852 (89.0\%) &
            6,697 (44.4\%) \\

            \# entities &
            109,976 &
            160,935 &
            132,460 \\

            \# images of entities &
            67,195 (61.1\%) &
            86,769 (53.9\%) &
            67,309 (50.8\%) \\

            Image ratio &
            67.26\% &
            57.41\% &
            50.16\% \\

            \hline
        \end{tabular}%
    }
\end{table}
\begin{table*}[t]
\centering
\caption{Main results on three MEL benchmarks. Bold and underlined values denote the best and second-best unsupervised results, respectively.}
\label{tab:main_results}
\renewcommand{\arraystretch}{1.2}
\resizebox{\textwidth}{!}{%
\begin{tabular}{|c|c|c|c|c|c|c|c|c|c|c|c|c|c|c|}
\hline
\multirow{2}{*}{\textbf{Category}} 
& \multirow{2}{*}{\textbf{M}} 
& \multirow{2}{*}{\textbf{Methods}} 
& \multicolumn{3}{c|}{\textbf{WikiMEL}~\cite{wang2022multimodal}} 
& \multicolumn{3}{c|}{\textbf{RichpediaMEL}~\cite{wang2022multimodal}} 
& \multicolumn{3}{c|}{\textbf{WikiDiverse}~\cite{wang2022wikidiverse}} 
& \multicolumn{3}{>{\columncolor{blue!8}}c|}{\textbf{Improvement}} \\
\cline{4-15}
& & 
& \multicolumn{1}{c|}{HIT@1} & \multicolumn{1}{c|}{HIT@5} & HIT@10 
& \multicolumn{1}{c|}{HIT@1} & \multicolumn{1}{c|}{HIT@5} & HIT@10
& \multicolumn{1}{c|}{HIT@1} & \multicolumn{1}{c|}{HIT@5} & HIT@10
& \multicolumn{1}{c|}{HIT@1} & \multicolumn{1}{c|}{HIT@5} & HIT@10 \\
\hline
 
\multirow{6}{*}{Supervised} 
 & T & ARNN~\cite{eshel2017named} & 32.0 & 45.8 & 56.6 & 31.2 & 39.3 & 45.9 & 22.4 & 50.5 & 68.4 & 47.85 & 42.78 & 33.46\\
 & T+V & MEL-HI~\cite{zhang2021attention} & 38.6 & 55.1 & 65.2 & 34.9 & 43.1 & 50.6 & 45.7 & 76.5 & 88.6 & 36.65
& 29.75 & 22.29 \\
 & T+V & GHMFC~\cite{wang2022multimodal} & 43.6 & 64.0 & 74.4 & 38.7 & 50.9 & 58.5  & 46.0 & 77.5 & 88.9  & 33.61 & 23.85 & 16.49 \\
 & T+V & DWE~\cite{song2024dual} & 44.7 & 65.9 & 80.8 & 67.6 & 97.1 & 98.6 & 47.5 & 81.3 & 92.0  & 23.11 & 6.55 & -0.04 \\
& T+V & UniMEL~\cite{liu2024unimel} & 94.1 & 97.2  & 98.4 & 94.8 & 97.9 & 98.3 & 92.9 & 97.0 & 99.5  & -16.29	&-8.30	&-7.53 \\
 & T+V & \(M^{3}EL\)(10\%)~\cite{hu2025multi} & 76.72 & 90.19 & 93.7 & 64.51 & 87.28 & 91.6 & 66.36 & 84.31 &  89.71 &  7.18 & 0.72 & -1.24 \\
 \hline

\multirow{4}{*}{Unsupervised} 
 & T & BERT~\cite{devlin2019bert} & 31.7 & 48.8 & 57.8 & 31.6 & 42.0 & 47.6 & 22.2 & 53.8 & 59.8 & 47.88 & 39.78 & 35.36 \\
 & T+V & CLIP~\cite{radford2021learning} & 40.7 & 56.0 & 64.6 & 38.1 & 54.5 & 62.4 & 34.4 & 59.7 & 62.2 & 38.65 & 31.25 & 27.36 \\
 & T+V & BM~\cite{gan2021multimodal} & 33.2 & 50.7 & 57.5 & 45.1 & 62.3 & 69.9 & 28.8 & 48.8 & 58.1 & 40.68 & 34.05 & 28.59\\
 & T+V & OpenMEL~\cite{zhu2025openmel} & \ul{75.1} & \ul{84.3} & \ul{85.9} & \underline{72.6} & \ul{81.1} & \ul{83.2} & \ul{72.4} & \ul{87.2} & \ul{90.3} & 4.28 & 4.87 & 4.75\\

 \hline
 Unsupervised & T+V & \textbf{\ModelName~(Ours)} & \textbf{84.99} & \textbf{91.76} & \textbf{92.86} & \textbf{74.4} & \textbf{87.48} & \textbf{89.56} & \textbf{73.55} & \textbf{87.97} & \textbf{91.2} & - & - & - \\ \hline
\end{tabular}%
}
\end{table*}
\noindent \textbf{Datasets.}
We evaluate \ModelName{} on three MEL benchmarks: WikiMEL, RichpediaMEL~\cite{wang2022multimodal}, and WikiDiverse~\cite{wang2022wikidiverse}.
As summarized in Table~\ref{tab:statistics_of_datasets}, WikiMEL and RichpediaMEL primarily contain person entities, WikiDiverse covers more diverse entity types, including persons, organizations, and locations.

\noindent \textbf{Baselines and Metrics.}
We compare \ModelName{} with four unsupervised methods: BM25~\cite{gan2021multimodal}, zero-shot BERT~\cite{devlin2019bert}, zero-shot CLIP~\cite{radford2021learning}, and OpenMEL~\cite{zhu2025openmel}.
We also include six supervised methods for reference: ARNN~\cite{eshel2017named}, MEL-HI~\cite{zhang2021attention}, GHMFC~\cite{wang2022multimodal}, DWE~\cite{song2024dual}, and M$^3$EL~\cite{hu2025multi}.
Performance is evaluated using Hit@1, Hit@5, and Hit@10, which measure whether the ground-truth entity appears within the top-$k$ predictions.

\noindent \textbf{Implementation details.}
Experiments are conducted on a workstation with an Intel Xeon Silver 4314 CPU, an NVIDIA RTX A5000 GPU, and 512 GB RAM. We use ChatGPT-4o in both the offline and online stages.
We set $K_{\mathrm{ch}}=250$, $K_{\mathrm{llm}}=300$,
$K_{\mathrm{ppr}}=20$, and $\lambda_1=0.75$.
The induced decision trees have a maximum depth of 5.

\subsection{Experimental Results}

\noindent \textbf{Exp-1: Overall Performance.}
As shown in Table~\ref{tab:main_results}, \ModelName{} consistently achieves the best performance among all unsupervised methods across the three benchmarks and evaluation metrics.
Compared with the strongest unsupervised baseline OpenMEL, \ModelName{} improves Hit@1 by 9.89, 1.80, and 1.15 points on WikiMEL, RichpediaMEL, and WikiDiverse, respectively.
The gains are also consistent on Hit@5 and Hit@10, with average improvements of 4.87 and 4.74 points across the three datasets.
Although supervised methods are included for reference,
MSR-MEL achieves the best performance among the evaluated
unsupervised methods and outperforms M$^3$EL trained
with 10\% labeled data on Hit@1 and Hit@5.

\begin{figure}[!t]
    \centering
    \includegraphics[width=1.0\linewidth]{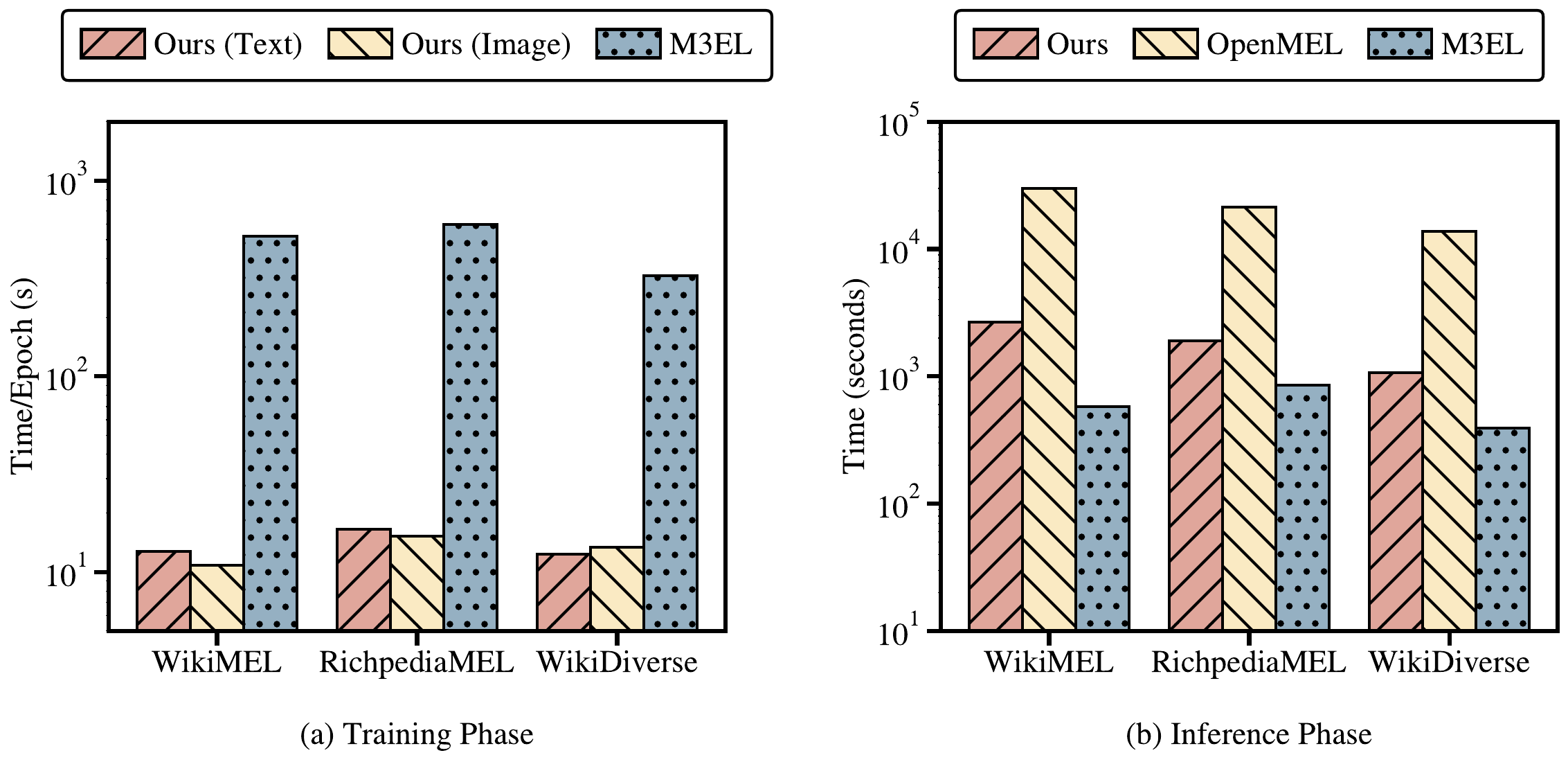}
    \caption{Runtime comparison on MEL benchmarks. (a) Average training time per epoch (log scale). (b) Online inference time (log scale). Additional LLM usage and cost statistics are reported in Appendix~\ref{app:llm_cost}.}
    \label{fig: efficiency}
\end{figure}

\noindent \textbf{Exp-2: Efficiency Evaluation.}
Figure~\ref{fig: efficiency} compares the training and inference efficiency of \ModelName{} on the three benchmarks.
During training, both the text-teacher and image-student branches require
approximately $10^{1}$ seconds per epoch, over one order of magnitude less than M$^3$EL.
This indicates that the teacher--student alignment remains lightweight
despite using graph-enhanced representations.

At inference time, M$^3$EL achieves the lowest inference time through direct end-to-end prediction, but requires labeled data for supervised training.
\ModelName{} is substantially faster than OpenMEL, reducing online inference time by approximately 11--13$\times$ across the three datasets.
These results show that the proposed online reasoning strategy achieves
strong accuracy without the high inference cost of explicit structural
optimization.

\noindent \textbf{Exp-3: Parameter Sensitivity Analysis.}
Performance is stable over reasonable $K_{\mathrm{llm}}$, $K_{\mathrm{ppr}}$, and $\lambda_1$ ranges (full analysis in Appendix~\ref{app:sensitivity}).

\noindent \textbf{Exp-4: Effect of LLM.}
\begin{table}[t]
\centering
\caption{Impact of LLM variants on semantic augmentation performance (Hit@1).}
\label{tab:llm_ablation}
\resizebox{\linewidth}{!}{%
\begin{tabular}{lccc}
\toprule
\textbf{LLM Variant} & \textbf{WikiMEL} & \textbf{RichpediaMEL} & \textbf{WikiDiverse} \\ 
\midrule
\textbf{Ours + ChatGPT-4o~\cite{achiam2023gpt}} & \textbf{84.99} & \textbf{74.40} & \textbf{73.55} \\ 
\midrule
\multicolumn{4}{l}{\textit{\textbf{A. Alternative LLMs for Semantic Expansion}}} \\
\quad OpenMEL + ChatGPT-4o~\cite{zhu2025openmel} & 75.10 & \ul{72.60} & \ul{72.40} \\
\quad Ours + DeepSeek~\cite{liu2024deepseek}       & \ul{82.73} & 71.29 & 65.10 \\
\quad OpenMEL + DeepSeek~\cite{zhu2025openmel}   & 67.20 & 69.30 & 63.80 \\
\quad Ours + LLaMA3-8B~\cite{touvron2023llama}          & 81.76 & 70.31 & 67.30 \\
\quad OpenMEL + LLaMA3-8B~\cite{zhu2025openmel}        & 69.80 & 65.60 & 67.10 \\
\bottomrule
\end{tabular}%
}
\end{table}
Table~\ref{tab:llm_ablation} compares \ModelName{}~and OpenMEL under three matched LLM backbones. \ModelName{}~consistently outperforms OpenMEL across all backbone--dataset combinations, indicating that its gains are not specific to a particular LLM. Within \ModelName{},
ChatGPT-4o achieves the best results on all three datasets. This advantage may stem from its stronger ability to generate precise and informative semantic descriptions from sparse mention contexts. 
The difference is particularly evident on WikiDiverse, where replacing ChatGPT-4o with DeepSeek or LLaMA3-8B decreases Hit@1 by 8.45 and 6.25 points, respectively. We therefore adopt ChatGPT-4o as the default backbone.

\begin{table}[t]
\centering
\caption{Hit@1 under different modality settings.}
\label{tab:modality}
\resizebox{\linewidth}{!}{%
\begin{tabular}{|l|l|c|c|c|}
\hline
\textbf{Setting} & \textbf{Method} & \textbf{WikiMEL} & \textbf{Richpedia} & \textbf{WikiDiverse} \\
\hline
\multirow{2}{*}{Full dataset}
& \cellcolor{gray!15}Ours & \cellcolor{gray!15}84.99 & \cellcolor{gray!15}74.40 & \cellcolor{gray!15}73.55 \\
& OpenMEL~\cite{zhu2025openmel} & 75.10 & 72.60 & 72.40 \\
\hline
\multirow{2}{*}{Complete modality}
& \cellcolor{gray!15}Ours & \cellcolor{gray!15}86.60 & \cellcolor{gray!15}76.20 & \cellcolor{gray!15}72.80 \\
& OpenMEL~\cite{zhu2025openmel} & 71.40 & 67.34 & 68.17 \\
\hline
\multirow{2}{*}{Missing modality}
& \cellcolor{gray!15}Ours & \cellcolor{gray!15}83.03 & \cellcolor{gray!15}71.97 & \cellcolor{gray!15}69.70 \\
& OpenMEL~\cite{zhu2025openmel} & 68.70 & 65.43 & 66.09 \\
\hline
\end{tabular}%
}
\end{table}
\noindent\textbf{Exp-5: Modality Robustness.}
Table~\ref{tab:modality} compares \ModelName{} and OpenMEL under three
modality settings. The full dataset corresponds to the default setting
and retains all samples. The complete-modality subset excludes
mention--entity pairs with missing images, whereas the missing-modality
subset retains only pairs in which either the mention or the entity lacks
an image. \ModelName{} consistently outperforms OpenMEL across all
datasets and settings. On the two filtered subsets, both methods perform
better with complete modalities than with missing modalities, while
\ModelName{} maintains a clear advantage in the missing-modality setting.
These results demonstrate the robustness of \ModelName{} to modality
incompleteness.

\subsection{Ablation study}
\begin{table}[t] 
\centering 
\caption{Ablation study of offline evidence synthesis and online evidence reasoning in \ModelName{} (Hit@1).} \label{tab:unified_ablation} 
\resizebox{0.46\textwidth}{!}{%
\begin{tabular}{lccc} 
\toprule \textbf{Variant} & \textbf{WikiMEL} & \textbf{RichpediaMEL} & \textbf{WikiDiverse} \\ \midrule 
\textbf{\ModelName{} (Full Model)} & \textbf{84.99} & \textbf{74.40} & \textbf{73.55} \\ \midrule 
\multicolumn{4}{l}{\textit{\textbf{A. Offline Evidence Synthesis}}} \\ 
\quad w/o LLM-enhanced Edges & 81.19 & 71.55 & 66.71 \\ 
\quad w/o Intra-type Edges (M--M, E--E) & 82.36 & 72.91 & 67.30 \\ 
\quad w/o Teacher--Student Alignment & 80.08 & 70.71 & 65.24 \\ 
\quad Text-only Group Evidence (w/o Student) & 83.29 & 73.88 & 68.23 \\ 
\quad Image-only Group Evidence (w/o Teacher) & 76.98 & 66.25 & 62.30 \\ 
\midrule \multicolumn{4}{l}{\textit{\textbf{B. Online Evidence Reasoning}}} \\ 
\quad w/o Instance-centric Evidence ($\mathbf{f}_{\text{inst}}$) & 53.87 & 49.30 & 46.12 \\ 
\quad w/o Group-level Evidence ($\mathbf{f}_{\text{group}}$) & 74.50 & 65.32 & 63.44 \\ 
\quad w/o Lexical Evidence ($\mathbf{f}_{\text{lex}}$) & 78.96 & 72.83 & 62.90 \\ 
\quad w/o Statistical Descriptors ($\mathbf{f}_{\text{stat}}$: $\mu, s_{\max}, \Delta$) & 83.29 & 72.90 & 68.71 \\ 
\quad w/o Modality-Availability Indicators ($\mathbf{f}_{\text{stat}}$: $\mathbb{I}_{\mathrm{img}}(m), \mathbb{I}_{\mathrm{img}}(e)$) & 82.78 & 73.12 & 66.80 \\ 
\quad w/o Prior Scoring ($s_{\text{prior}}$) & 76.66 & 70.34 & 63.90 \\ 
\quad w/o LLM-driven Evidence Reasoning ($\pi_{\text{LLM}}$) & 81.43 & 72.34 & 66.37 \\ \bottomrule 
\end{tabular}%
} 
\end{table}
\noindent \textbf{Exp-7: The Effect of Offline Evidence Synthesis.}
Table~\ref{tab:unified_ablation} shows that all offline components
contribute to performance. Removing LLM-enhanced or intra-type edges
consistently reduces Hit@1, confirming the benefit of semantic expansion
and neighborhood connectivity. Disabling teacher--student alignment
causes larger drops, particularly on WikiDiverse. The text-only variant
remains competitive but underperforms the full model, whereas the
image-only variant performs substantially worse. These results support
the joint use of contextualized graph construction and asymmetric
cross-modal alignment.
The full model outperforms mean-neighbor/PPR aggregation and GCN without teacher--student alignment (Appendix~\ref{app:graph_alternatives}).

\noindent\textbf{Exp-8: Effect of Online Evidence Reasoning.}
Table~\ref{tab:unified_ablation} shows that the online evidence sources are complementary.
Instance-centric evidence contributes most, while group-level, lexical, statistical, and modality-availability evidence also improve performance, especially on WikiDiverse.
Removing either prior scoring or LLM-driven reasoning further degrades performance.

We further analyze the two online stages.
Candidate Recall (oracle upper-bound Hit@1) reaches 99.45/98.79/98.31\% on WikiMEL/RichpediaMEL/WikiDiverse.
Within the same candidate sets, query-specific LLM reranking improves prior Hit@1 from 81.43/72.34/66.37\% to 84.99/74.40/73.55\%, showing that retrieval ensures broad coverage while reasoning improves the final ranking.

We also compare alternative evidence-combination strategies.
Mean fusion, RRF, a conventional decision tree, and a global LLM tree all underperform prior scoring, whereas the query-specific LLM tree improves it (Appendix~\ref{app:reasoning_alternatives}).
Across five reranking runs, performance remains stable at $84.99\pm0.06$/$74.40\pm0.10$/$73.55\pm0.13$ Hit@1, with no parsing or execution failures (Appendix~\ref{app:stability}).
Further subset and low-label supervised-tree analyses are provided in Appendices~\ref{app:subsets} and~\ref{app:supervised_trees}.

\section{Related Work}
\label{sec:related}
\subsection{Multimodal Entity Linking}
Entity linking (EL) maps textual mentions to entities in a KB~\cite{shen2014entity,rao2012entity}.
Earlier methods relied on lexical and statistical features, while later
approaches adopted neural encoders and large pre-trained language models.
Representative methods include dense retrieval models such as
BLINK~\cite{wu2020scalable} and generative models such as
GENRE~\cite{de2020autoregressive}.

Multimodal entity linking (MEL) extends EL by incorporating visual
information for disambiguation.
Supervised methods learn from annotated mention--entity pairs through
multimodal alignment or semantic enrichment, including
M$^3$EL~\cite{hu2025multi}, UniMEL~\cite{liu2024unimel},
KGMEL~\cite{kim2025kgmel}, I2CR~\cite{liu2025i2cr}, and
DeepMEL~\cite{wang2026deepmel}.
Unsupervised methods instead use pre-trained models such as
BERT~\cite{devlin2019bert} and CLIP~\cite{radford2021learning} to rank
candidates without task-specific annotations.
OpenMEL~\cite{zhu2025openmel} further introduces contextual expansion
and structural organization, but its decisions remain largely driven by
similarity-based evidence.
In contrast, \ModelName{} combines instance-level and group-level
evidence and explicitly reasons over evidence reliability and consistency.

\subsection{Graph Neural Network}
Graph neural networks (GNNs) learn contextualized node representations
by aggregating information from graph neighborhoods.
Representative architectures include GCN~\cite{kipf2016semi},
GraphSAGE~\cite{hamilton2017inductive}, and
GAT~\cite{velivckovic2017graph}.
Their ability to capture relational context makes them suitable for MEL,
where individual textual or visual signals may be sparse or noisy.
In this work, we adopt GCN for its simplicity and efficiency and use it
to derive group-level representations from the constructed
mention--entity graph.

\section{Conclusion}
\label{sec:conclusion}
In this paper, we propose \ModelName{}, an LLM-guided multi-perspective
evidence synthesis and reasoning framework for unsupervised multimodal
entity linking.
\ModelName{} integrates instance-centric, group-level, lexical, and
statistical evidence, with group-level representations learned from
LLM-enhanced graphs through asymmetric teacher--student alignment.
During inference, an LLM-induced decision tree reasons over this
heterogeneous evidence to rerank candidate entities.
Experiments on three MEL benchmarks show that \ModelName{} consistently
outperforms state-of-the-art unsupervised methods, demonstrating the
effectiveness of evidence synthesis and reasoning.

\section*{Limitations}
\ModelName{} relies on pretrained LLMs for semantic augmentation and decision-tree induction, which introduces additional cost and sensitivity to model updates. Pretrained knowledge may also create potential benchmark contamination risks. Moreover, our evaluation is limited to three English Wikipedia-based benchmarks, and its generalization to other languages, domains, and more severe modality incompleteness remains to be studied.


\bibliography{main}

\appendix
\section{Appendix}
\label{sec:appendix}
\subsection{Comparison with Existing MEL Methods}
\label{app:method_comparison}

Table~\ref{tab:comparison} compares representative MEL
methods along three dimensions: supervision requirements, the
types of evidence considered, and the role of LLMs. Most existing
approaches primarily rely on instance-centric evidence, while a
few additionally incorporate lexical cues or employ LLMs for
context enhancement, candidate selection, control, or agent-based
interaction. In contrast, MSR-MEL is label-free and jointly
integrates instance-centric, group-level, lexical, and statistical
evidence, with LLMs supporting both semantic augmentation and
decision-tree induction. 
Overall, MSR-MEL is distinguished by its label-free integration of multi- perspective evidence, together with LLM-driven evidence reasoning.

\subsection{Evaluation Scope and Modality Handling}
\label{app:evaluation_scope}

\noindent\textbf{Matching scope.}
Following the closed-world protocols of the evaluated MEL benchmarks, each mention is associated with a ground-truth entity in the KB, and MSR-MEL returns the highest-scoring entity in the retrieved candidate set $\mathcal{C}(m)$.
Open-world NIL detection is not evaluated in this work, but can be incorporated by returning \textsc{nil} when $\max_{e\in\mathcal{C}(m)} s_{\mathrm{LLM}}(m,e)<\gamma$, where $\gamma$ is a predefined decision threshold.

\noindent\textbf{Modality availability.}
The current benchmarks provide textual information for mentions and entities, while visual inputs may be missing or noisy. 
More generally, \ModelName{} can be extended to settings where either the textual or visual modality is unavailable.
For a missing modality, its input embedding can be replaced with a zero vector, while the corresponding node is retained in the shared graph to preserve the graph topology and enable message passing through the available modality. 
Modality-availability indicators are included in the statistical evidence, allowing the reasoning stage to distinguish an unavailable modality from an observed low-similarity evidence.

\subsection{Notation Summary}
\label{app:symbols}
\begin{table}[!htbp]
\centering
\caption{Comparison of representative MEL methods.
$E_{\mathrm{inst}}$, $E_{\mathrm{group}}$, $E_{\mathrm{lex}}$,
and $E_{\mathrm{stat}}$ denote instance-centric, group-level,
lexical, and statistical evidence.}
\label{tab:comparison}
\resizebox{\linewidth}{!}{%
\begin{tabular}{l c c c c c c}
\toprule
\textbf{Approach} & \textbf{Label-free?} &
\multicolumn{4}{c}{\textbf{Evidence Types}} & \textbf{LLM Role} \\
\cmidrule(lr){3-6}
 & & E$_\text{inst}$ & E$_\text{group}$ & E$_\text{lex}$ & E$_\text{stat}$ & \\
\midrule
\quad DWE~\cite{song2024dual}       
& \ding{56} 
& \ding{52} & \ding{56}  & \ding{56}  & \ding{56} 
& \ding{56} \\
\quad M$^3$EL~\cite{hu2025multi}    
& \ding{56} 
& \ding{52} & \ding{56}  & \ding{56}  & \ding{56} 
& \ding{56} \\
\quad UniMEL~\cite{liu2024unimel}                        
& \ding{56} 
& \ding{52} & \ding{56}  & \ding{56}  & \ding{56} 
& Context Enhancer \\
\quad KGMEL~\cite{kim2025kgmel}                         
& \ding{56} 
& \ding{52} & \ding{56}  & \ding{56}  & \ding{56} 
& Context Enhancer + Selector  \\
\quad I2CR~\cite{liu2025i2cr}                          
& \ding{56} 
& \ding{52} & \ding{56}  & \ding{52}  & \ding{56} 
& Controller        \\
\midrule
\quad OpenMEL~\cite{zhu2025openmel} 
& \ding{52} 
& \ding{52} & \ding{56}  & \ding{52}  & \ding{56} 
& Context Enhancer \\
\quad BERT~\cite{devlin2019bert}   
& \ding{52} 
& \ding{52} & \ding{56}  & \ding{56}  & \ding{56} 
& \ding{56}       \\
\quad DeepMEL~\cite{wang2026deepmel}                       
& \ding{52} 
& \ding{52} & \ding{56}  & \ding{56}  & \ding{56} 
& Agents            \\
\midrule
\quad Ours (\ModelName)             
& \ding{52} 
& \ding{52} & \ding{52}  & \ding{52}  & \ding{52} 
& Semantic Enhancer + Decision-Tree Inducer  \\
\bottomrule
\end{tabular}%
}
\end{table}
\setlength{\textfloatsep}{2pt}
Table~\ref{tab:symbols} summarizes the main notation used in the offline evidence-synthesis and online evidence-reasoning stages of \ModelName{}.

\begin{table}[t]
\centering
\caption{\textbf{Symbols and descriptions.}}
\label{tab:symbols}
\renewcommand{\arraystretch}{1.15}
\resizebox{\columnwidth}{!}{%
\begin{tabular}{|l|l|l|}
\hline
\rowcolor[HTML]{C0C0C0}
\textbf{Category} & \textbf{Symbol} & \textbf{Description} \\ \hline
\multirow{4}{*}{Evidence Types}
& $\mathbf{f}_{\text{inst}}(m,e)$ & Instance-centric evidence \\
\cline{2-3}
& $\mathbf{f}_{\text{group}}(m,e)$ & Group-level evidence \\
\cline{2-3}
& $\mathbf{f}_{\text{lex}}(m,e)$ & Lexical evidence \\
\cline{2-3}
& $\mathbf{f}_{\text{stat}}(m,e)$ & Statistical evidence \\
\hline
\multirow{6}{*}{Sets \& Graph}
& $\mathcal{M}, \mathcal{E}$ & Mention set and entity set \\ \cline{2-3}
& $\mathcal{V}=\mathcal{M}\cup\mathcal{E}$ & Node set (mentions + entities) \\ \cline{2-3}
& $E^{(0)}$ & CLIP-gated candidate edge set \\ \cline{2-3}
& $E^{(1)}$ & LLM-enhanced candidate edge set \\ \cline{2-3}
& $E_{\text{graph}}$ & Final edge set ($E^{(0)} \cup E^{(1)}$) \\ \cline{2-3}
& $\mathcal{G}=(\mathcal{V},E_{\text{graph}})$ & Contextualized multimodal graph \\ \hline

\multirow{4}{*}{Subgraphs}
& $\mathcal{G}_{v_i}$ & PPR subgraph centered at node $v_i$ \\ \cline{2-3}
& $K_{\mathrm{ppr}}$
& Number of nodes retained in each PPR subgraph \\ \cline{2-3}
& $\mathcal{G}^{\text{text}}_{v_i}$ & Text-view subgraph with features $\mathbf{t}_v$ \\ \cline{2-3}
& $\mathcal{G}^{\text{img}}_{v_i}$ & Image-view subgraph with features $\mathbf{v}_v$ \\ \hline


\multirow{3}{*}{GNN Representations}
& $f_{\theta_q}(\cdot)$
& GNN encoder for modality $q\in\{\mathrm{txt},\mathrm{img}\}$ \\ \cline{2-3}

& $\mathbf{h}^{q}_{v_i}$
& Contextualized representation of node $v_i$ \\ \cline{2-3}

& $\mathbf{s}^{q}_{v_i}$
& Mean-pooled summary of subgraph $\mathcal{G}_{v_i}$ \\ \hline
\multirow{5}{*}{Evidence Reasoning}
& $\mathcal{C}(m)$
& Candidate entity set for mention $m$ \\ \cline{2-3}

& $s_{\text{prior}}(m,e)$
& Initial score aggregated from retrieval channels \\ \cline{2-3}

& $\mathcal{P}(m,e)$
& Decision path for mention--entity pair $(m,e)$ \\ \cline{2-3}

& $\Delta_n$
& Score adjustment associated with branch $n$ \\ \cline{2-3}

& $s_{\text{LLM}}(m,e)$
& Final score after LLM-driven re-ranking \\ \hline

\multirow{4}{*}{Loss \& Params}
& $\mathcal{L}_{\text{struct}}$ & Structural contrastive loss \\ \cline{2-3}
& $\mathcal{L}_{\text{xmodal}}$ & Cross-modal InfoNCE loss \\ \cline{2-3}
& $\lambda_1$
& Weight of the cross-modal alignment loss \\
& $\tau, \eta$ & Temperature and margin hyperparameters \\ \hline

\end{tabular}
}
\end{table}

\subsection{Lexical Evidence Implementation}
\label{app:lexical_evidence}
We compute lexical evidence using case-insensitive Gestalt Pattern Matching between the mention surface form $m$ and candidate entity name $e$. The algorithm recursively identifies matching subsequences in the two strings. 
Let $M(m,e)$ denote the total length of the matched subsequences. The lexical similarity is defined as:
\begin{equation}
\label{eq:lexical_similarity}
    s_{\mathrm{lex}}(m,e)
    =
    \frac{2M(m,e)}{|m|+|e|},
\end{equation}
where $s_{\mathrm{lex}}(m,e)\in[0,1]$. Higher values indicate greater character-level overlap between the mention and entity name.

\subsection{Complexity Analysis}
\label{app:complexity}
Let $N=|\mathcal{V}|$, $M=|E_{\mathrm{graph}}|$, $d$ be the hidden
dimension, $L$ the number of GNN layers, and $R$ the number of training
epochs for each GNN. Let
$N_{\mathrm{aug}}=|\mathcal{M}|+|\mathcal{E}_k|$ denote the number of
nodes augmented by the LLM. The offline complexity is
\begin{equation}
\small
O\!\left(
NT_{\mathrm{enc}}
+
N_{\mathrm{aug}}T_{\mathrm{llm}}^{\mathrm{off}}
+
T_{\mathrm{graph}}
+
T_{\mathrm{ppr}}
+
RLMd
\right),
\end{equation}
where $T_{\mathrm{graph}}$ and $T_{\mathrm{ppr}}$ denote the costs of
graph construction and PPR subgraph extraction.

For each query mention, candidate selection performs eight dense
retrievals and one lexical retrieval, requiring
$O(8T_{\mathrm{dense}}+T_{\mathrm{lex}})$.
Let $C=|\mathcal{C}(m)|$ and $D$ denote the candidate-set size and the
maximum decision-tree depth, respectively.
Generating one decision tree and applying it to all candidates costs
$O(T_{\mathrm{llm}}^{\mathrm{on}}+CD)$.
Therefore, the online complexity per query is
\begin{equation}
\small
O\!\left(
8T_{\mathrm{dense}}
+
T_{\mathrm{lex}}
+
T_{\mathrm{llm}}^{\mathrm{on}}
+
CD
\right).
\end{equation}
The offline computation is performed once and amortized over all queries,
while the online cost depends on the retrieved candidate set rather than
the full entity space.

\subsubsection{Measured LLM Cost}
\label{app:llm_cost}
Table~\ref{tab:llm_cost} reports API usage across offline semantic augmentation and online tree induction, together with stage-specific runtime. Offline LLM augmentation is a one-time cost. The per-epoch GNN training runtime in Figure~\ref{fig: efficiency} does not include this augmentation cost. API charges are the recorded experimental costs.

\begin{table}[!htbp]
\centering
\resizebox{\linewidth}{!}{%
\begin{tabular}{lccc}
\toprule
Cost item & WikiMEL & RichpediaMEL & WikiDiverse\\
\midrule
API calls & 12,449 & 16,449 & 13,583\\
Tokens & 2,649,385 & 3,486,591 & 2,879,100\\
API cost (USD) & 0.691 & 0.910 & 0.751\\
Offline LLM time (h) & 4.95 & 6.51 & 5.38\\
Online inference (min) & 44.60 & 31.80 & 17.82\\
\bottomrule
\end{tabular}}
\caption{LLM API usage and stage-specific runtime.}
\label{tab:llm_cost}
\end{table}

\subsection{Additional Theoretical Analysis}
\label{app:theoretical_analysis}

In this section, we present two idealized analyses that support
the core design choices of \ModelName{}. 
Theorem~\ref{theorem:evidence} supports the use of
multi-perspective evidence synthesis by showing that
minimum-risk fusion is no worse than relying on any individual
evidence channel. Theorem~\ref{theorem:ts} supports the
asymmetric teacher--student alignment by relating the image
student's approximation error to the teacher--student
discrepancy under a teacher-stability assumption. These results
provide theoretical motivation for the corresponding designs
rather than end-to-end performance guarantees.

\begin{theorem}
\label{theorem:evidence}
Let $y^*=y^*(m,e)\in\mathbb{R}$ denote the latent
compatibility between a mention $m$ and an entity $e$.
Suppose that the $k$-th evidence channel provides an
observation
\[
z_k = y^* + \varepsilon_k,
\qquad k\in\{1,\ldots,K\},
\]
where $\mathbb{E}[\varepsilon_k]=0$. Let
$\boldsymbol{\varepsilon}
=[\varepsilon_1,\ldots,\varepsilon_K]^\top$
have a positive-definite covariance matrix
$\Sigma\in\mathbb{R}^{K\times K}$.

For a weight vector
\[
\mathbf{w}\in\Delta_K
=
\left\{
\mathbf{w}\in\mathbb{R}^K
\mid
w_k\ge 0,\;
\mathbf{1}^\top\mathbf{w}=1
\right\},
\]
define the linear fusion estimator
\[
\theta_{\mathbf{w}}(\mathbf{z})
=
\mathbf{w}^\top\mathbf{z}
\]
and its estimation risk as
\[
\mathcal{R}(\theta_{\mathbf{w}})
=
\mathbb{E}
\left[
\left(
\theta_{\mathbf{w}}(\mathbf{z})-y^*
\right)^2
\right].
\]

Let
\[
\mathbf{w}^*
=
\arg\min_{\mathbf{w}\in\Delta_K}
\mathcal{R}(\theta_{\mathbf{w}})
\]
and define the minimum-risk multi-perspective estimator as
$\theta^{mp}=\theta_{\mathbf{w}^*}$. Let
$\theta_k^{sp}=\theta_{\mathbf{e}_k}$ denote the
single-perspective estimator using only the $k$-th evidence
channel, where $\mathbf{e}_k$ is the $k$-th standard basis
vector. Then
\[
\mathcal{R}(\theta^{mp})
\le
\mathcal{R}(\theta_k^{sp}),
\qquad
\forall k\in\{1,\ldots,K\}.
\]
Moreover, if $\mathbf{w}^*$ assigns positive weights to at
least two evidence channels, then
\[
\mathcal{R}(\theta^{mp})
<
\mathcal{R}(\theta_k^{sp}),
\qquad
\forall k\in\{1,\ldots,K\}.
\]
\end{theorem}

\noindent
\textbf{Proof of Theorem~\ref{theorem:evidence}.}
Since
$\mathbf{z}=y^*\mathbf{1}+\boldsymbol{\varepsilon}$
and $\mathbf{1}^\top\mathbf{w}=1$, the estimation error of
$\theta_{\mathbf{w}}$ is
\[
\theta_{\mathbf{w}}(\mathbf{z})-y^*
=
\mathbf{w}^\top
\left(
y^*\mathbf{1}+\boldsymbol{\varepsilon}
\right)
-y^*
=
\mathbf{w}^\top\boldsymbol{\varepsilon}.
\]
Therefore,
\[
\begin{aligned}
\mathcal{R}(\theta_{\mathbf{w}})
&=
\mathbb{E}
\left[
\left(
\mathbf{w}^\top\boldsymbol{\varepsilon}
\right)^2
\right] \\
&=
\mathbf{w}^\top
\mathbb{E}
\left[
\boldsymbol{\varepsilon}
\boldsymbol{\varepsilon}^\top
\right]
\mathbf{w} \\
&=
\mathbf{w}^\top\Sigma\mathbf{w}.
\end{aligned}
\]

For every evidence channel $k$, the one-hot vector
$\mathbf{e}_k$ belongs to $\Delta_K$. It is therefore a
feasible solution of the minimum-risk fusion problem. By the
definition of $\mathbf{w}^*$,
\[
\mathcal{R}(\theta^{mp})
=
\mathcal{R}(\theta_{\mathbf{w}^*})
\le
\mathcal{R}(\theta_{\mathbf{e}_k})
=
\mathcal{R}(\theta_k^{sp}),
\]
for every $k\in\{1,\ldots,K\}$.

Because $\Sigma$ is positive definite, the function
$\mathbf{w}^\top\Sigma\mathbf{w}$ is strictly convex.
Consequently, the minimizer $\mathbf{w}^*$ over the convex
set $\Delta_K$ is unique. If $\mathbf{w}^*$ assigns positive
weights to at least two channels, then
$\mathbf{w}^*\neq\mathbf{e}_k$ for every $k$. Strict
convexity and uniqueness of the minimizer therefore imply
\[
\mathcal{R}(\theta_{\mathbf{w}^*})
<
\mathcal{R}(\theta_{\mathbf{e}_k}),
\qquad
\forall k\in\{1,\ldots,K\}.
\]
This proves the theorem.

\begin{theorem}
\label{theorem:ts}
Let $\mathcal{H}\subseteq\mathbb{R}^{d}$ be the shared
representation space. For a mention or entity node $v$, let
$\mathcal{L}_{\mathrm{MEL}}(\mathbf{h};v)$ denote an
idealized representation loss that measures how well
$\mathbf{h}$ preserves the semantic and neighborhood
information of $v$ for MEL under complete and noise-free
information. Assuming that its minimum exists, define an ideal
representation by
\[
\mathbf{h}_{v}^{*}
\in
\arg\min_{\mathbf{h}\in\mathcal{H}}
\mathcal{L}_{\mathrm{MEL}}(\mathbf{h};v).
\]

Let $\mathbf{h}_{v}^{\mathrm{txt}}$ and
$\mathbf{h}_{v}^{\mathrm{img}}$ denote the text-based teacher
and image-based student representations, respectively. If the
teacher is a stable approximation to the ideal representation,
such that
\[
\left\|
\mathbf{h}_{v}^{\mathrm{txt}}
-
\mathbf{h}_{v}^{*}
\right\|_2
\le
\delta_t
\]
for some $\delta_t\ge 0$, then
\[
\left\|
\mathbf{h}_{v}^{\mathrm{img}}
-
\mathbf{h}_{v}^{*}
\right\|_2^2
\le
2
\left\|
\mathbf{h}_{v}^{\mathrm{img}}
-
\mathbf{h}_{v}^{\mathrm{txt}}
\right\|_2^2
+
2\delta_t^2.
\]
Thus, under the teacher-stability assumption, reducing the
teacher--student discrepancy tightens an upper bound on the
student's approximation error relative to the ideal
representation.
\end{theorem}

\noindent
\textbf{Proof of Theorem~\ref{theorem:ts}.}
By the triangle inequality in the shared representation space,
\[
\left\|
\mathbf{h}_{v}^{\mathrm{img}}
-
\mathbf{h}_{v}^{*}
\right\|_2
\le
\left\|
\mathbf{h}_{v}^{\mathrm{img}}
-
\mathbf{h}_{v}^{\mathrm{txt}}
\right\|_2
+
\left\|
\mathbf{h}_{v}^{\mathrm{txt}}
-
\mathbf{h}_{v}^{*}
\right\|_2.
\]
Squaring both sides and applying
$(a+b)^2\le 2a^2+2b^2$ gives
\[
\begin{aligned}
\left\|
\mathbf{h}_{v}^{\mathrm{img}}
-
\mathbf{h}_{v}^{*}
\right\|_2^2
&\le
2
\left\|
\mathbf{h}_{v}^{\mathrm{img}}
-
\mathbf{h}_{v}^{\mathrm{txt}}
\right\|_2^2 \\
&\quad+
2
\left\|
\mathbf{h}_{v}^{\mathrm{txt}}
-
\mathbf{h}_{v}^{*}
\right\|_2^2.
\end{aligned}
\]
Using the teacher-stability assumption
\[
\left\|
\mathbf{h}_{v}^{\mathrm{txt}}
-
\mathbf{h}_{v}^{*}
\right\|_2
\le
\delta_t,
\]
we obtain
\[
\left\|
\mathbf{h}_{v}^{\mathrm{img}}
-
\mathbf{h}_{v}^{*}
\right\|_2^2
\le
2
\left\|
\mathbf{h}_{v}^{\mathrm{img}}
-
\mathbf{h}_{v}^{\mathrm{txt}}
\right\|_2^2
+
2\delta_t^2.
\]
This proves the theorem.

For unit-normalized representations, the alignment term can
also be written as
\[
\left\|
\mathbf{h}_{v}^{\mathrm{img}}
-
\mathbf{h}_{v}^{\mathrm{txt}}
\right\|_2^2
=
2
\left(
1-
\operatorname{sim}
\left(
\mathbf{h}_{v}^{\mathrm{img}},
\mathbf{h}_{v}^{\mathrm{txt}}
\right)
\right).
\]
Hence, encouraging higher similarity between the matched
teacher and student representations reduces the
alignment-related term in the bound. The theorem does not
claim that teacher--student alignment alone recovers
$\mathbf{h}_{v}^{*}$; rather, it shows that, when the textual
teacher is stable, reducing the student--teacher discrepancy
controls an upper bound on the student's approximation error.

\subsection{Case Study}
\begin{figure}[!t]
    \centering
    \includegraphics[width=0.95\linewidth]{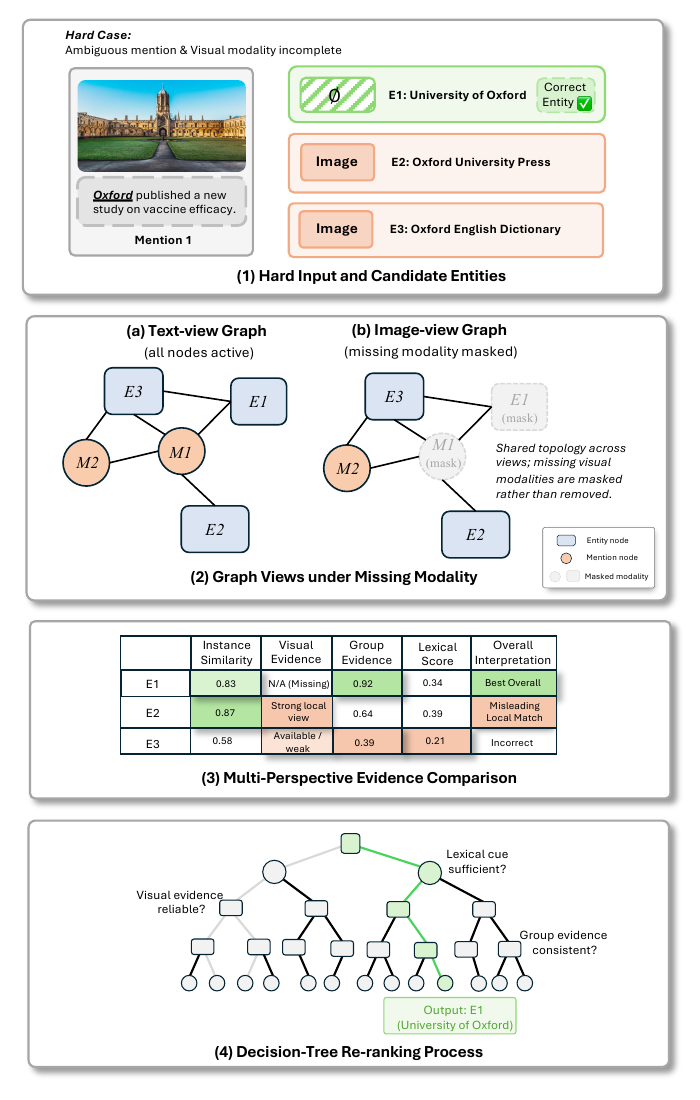}
    \caption{Case study of \ModelName{} on a hard example with missing visual modality.}
    \label{fig:qualitative_case_study}
\end{figure}
Figure~\ref{fig:qualitative_case_study} presents a challenging case with the ambiguous mention \textit{Oxford} and incomplete visual information.
The ground-truth entity, \textit{University of Oxford} (E1), has no visual modality, while the distractor \textit{Oxford University Press} (E2) obtains higher instance and lexical scores, as shown in Figure~\ref{fig:qualitative_case_study}(3). 
Therefore, relying only on local matching may incorrectly favor E2.

Nevertheless, \ModelName{} ranks E1 first. As shown in
Figure~\ref{fig:qualitative_case_study}(2), missing visual modalities are
masked rather than removing the corresponding nodes, preserving the
shared graph topology across views. E1 consequently receives stronger
group-level evidence than E2. The decision-tree reasoning process in
Figure~\ref{fig:qualitative_case_study}(4) further evaluates visual
reliability, lexical sufficiency, and group-level consistency, identifying
E2 as a misleading local match and selecting E1. This case demonstrates
the robustness of \ModelName{} to ambiguous mentions and incomplete
modalities.

\subsection{Algorithm}
\label{app:algorithm}
\FloatBarrier

\newcommand{\algcommentstyle}[1]{%
  \footnotesize\ttfamily\color{blue!65!black}#1%
}

{
\SetCommentSty{algcommentstyle}

\begin{algorithm}[!ht]
\caption{Offline Group-Level Evidence Synthesis}
\label{alg:offline_synthesis_alignment}
\DontPrintSemicolon

\KwIn{Mention and entity sets $\mathcal{M},\mathcal{E}$;
text and image features $\{\mathbf{t}_v,\mathbf{v}_v\}$;
$\delta_{\mathrm{gate}},K_{\mathrm{llm}},K_{\mathrm{ppr}},\lambda_1$}
\KwOut{Graph $\mathcal{G}$ and encoders
$f_{\theta_{\mathrm{txt}}},f_{\theta_{\mathrm{img}}}$}

\tcp{\textbf{Part I: Contextualized graph construction}}
$\mathcal{V}\leftarrow\mathcal{M}\cup\mathcal{E}$\;

Construct
$E^{(0)}
\leftarrow
\{(v_i,v_j)\mid
s_{\mathrm{gate}}(v_i,v_j)\geq\delta_{\mathrm{gate}}\}$
using Eq.~\eqref{eq:gated_sim}\;

\ForEach{$x\in\mathcal{M}\cup\mathcal{E}_k$}{
    Generate $\tilde{\mathbf{t}}_x$ using Eq.~\eqref{eq:llm_clip_aug}\;
}

\ForEach{$v_i\in\mathcal{V}$}{
    Compute $s_{\mathrm{llm}}(v_i,\cdot)$ using
    Eq.~\eqref{eq:llm_fusion}\;
    Add its Top-$K_{\mathrm{llm}}$ neighbors to $E^{(1)}$\;
}

$E_{\mathrm{graph}}\leftarrow E^{(0)}\cup E^{(1)}$;
$\mathcal{G}\leftarrow(\mathcal{V},E_{\mathrm{graph}})$\;

\tcp{\textbf{Part II: Asymmetric teacher--student alignment}}
Extract a Top-$K_{\mathrm{ppr}}$ PPR subgraph
$\mathcal{G}_{v_i}$ for each $v_i\in\mathcal{V}$\;

\tcp{Train the text-based teacher}
\While{teacher not converged}{
    Compute $\mathbf{h}^{\mathrm{txt}}_{v_i}$ and
    $\mathbf{s}^{\mathrm{txt}}_{v_i}$ using
    Eq.~\eqref{eq:graph_encoding_compact}\;
    Minimize
    $\mathcal{L}_{\mathrm{text}}
    =\mathcal{L}_{\mathrm{struct}}^{\mathrm{txt}}$
    using Eq.~\eqref{eq:struct_loss}\;
}
Freeze $f_{\theta_{\mathrm{txt}}}$\;

\tcp{Distill the image-based student}
\While{student not converged}{
    Compute frozen teacher representations
    $\mathbf{h}^{\mathrm{txt}}_{v_i}$ and student representations
    $\mathbf{h}^{\mathrm{img}}_{v_i}$\;
    Minimize
    $\mathcal{L}_{\mathrm{img}}
    =
    \mathcal{L}_{\mathrm{struct}}^{\mathrm{img}}
    +
    \lambda_1\mathcal{L}_{\mathrm{xmodal}}$
    using Eq.~\eqref{eq:img_loss}\;
}

\Return{$\mathcal{G},
f_{\theta_{\mathrm{txt}}},
f_{\theta_{\mathrm{img}}}$}\;
\end{algorithm}
}
{
\SetCommentSty{algcommentstyle}

\begin{algorithm}[!t]
\caption{Online Multi-Perspective Evidence Reasoning}
\label{alg:reranking}
\DontPrintSemicolon

\KwIn{
Query mention $m$;
entity set $\mathcal{E}$;
instance-centric and group-level similarity functions $\{s_t^q\}$;
lexical score $s_{\mathrm{lex}}$;
LLM $\mathcal{L}$
}
\KwOut{Predicted entity $\hat{e}$}

\vspace{2pt}
\tcp{\textbf{Step 1: Candidate Selection}}
Initialize $\mathcal{C}(m)\leftarrow\emptyset$\;

\ForEach{$t\in\{\mathrm{inst},\mathrm{group}\}$}{
    \ForEach{$q\in\{\mathrm{TT},\mathrm{TV},\mathrm{VT},\mathrm{VV}\}$}{
        $\mathcal{C}(m)\leftarrow
        \mathcal{C}(m)\cup
        \operatorname{TopK}_{e\in\mathcal{E}}
        \bigl(s_t^q(m,e),K_{\mathrm{ch}}\bigr)$\;
    }
}

$\mathcal{C}(m)\leftarrow
\mathcal{C}(m)\cup
\operatorname{TopK}_{e\in\mathcal{E}}
\bigl(s_{\mathrm{lex}}(m,e),K_{\mathrm{ch}}\bigr)$\;

\vspace{2pt}
\tcp{\textbf{Step 2: Prior Scoring and Evidence Construction}}
\ForEach{$e\in\mathcal{C}(m)$}{
    Compute $s_{\mathrm{prior}}(m,e)$ from normalized retrieval scores\;
    Construct the evidence vector $\mathbf{f}(m,e)$\;
}

\vspace{2pt}
\tcp{\textbf{Step 3: LLM-Induced Decision-Tree Re-ranking}}
$\mathcal{T}_{\mathrm{LLM}}
\leftarrow
\mathcal{L}\!\left(
\left\{
s_{\mathrm{prior}}(m,e),
\mathbf{f}(m,e)
\right\}_{e\in\mathcal{C}(m)}
\right)$\;

\ForEach{$e\in\mathcal{C}(m)$}{
    $\mathcal{P}(m,e)
    \leftarrow
    \operatorname{Traverse}
    \bigl(\mathcal{T}_{\mathrm{LLM}},\mathbf{f}(m,e)\bigr)$\;

    $s_{\mathrm{LLM}}(m,e)
    \leftarrow
    s_{\mathrm{prior}}(m,e)
    +
    \displaystyle\sum_{n\in\mathcal{P}(m,e)}\Delta_n$\;
}

\Return{
$\hat{e}
=
\displaystyle\arg\max_{e\in\mathcal{C}(m)}
s_{\mathrm{LLM}}(m,e)$
}\;

\end{algorithm}
}
We provide the detailed procedures of the offline and online stages of \ModelName{} in Algorithm~\ref{alg:offline_synthesis_alignment} and Algorithm~\ref{alg:reranking}, respectively.

Algorithm~\ref{alg:offline_synthesis_alignment} summarizes offline group-level evidence synthesis. In Part I, an initial graph is constructed using gated multimodal similarity. LLM-augmented textual representations are then generated for mentions and selected entities, and their Top-$K_{\mathrm{llm}}$ neighbors are added to obtain the final contextualized graph. 
In Part II, a Top-$K_{\mathrm{ppr}}$ PPR subgraph is extracted for each node. The text-based teacher is first optimized using the structural objective and subsequently frozen. The image-based student is then trained with both the structural objective and the cross-modal alignment loss, producing the group-level text and image encoders.

Algorithm~\ref{alg:reranking} describes online multi-perspective evidence
reasoning. First, a high-recall candidate set is constructed from nine
retrieval channels: eight neural channels obtained from four modality
combinations over instance-centric and group-level representations, and
one lexical channel. Second, the normalized retrieval scores are used to
compute the prior score $s_{\mathrm{prior}}(m,e)$ and construct the
multi-perspective evidence vector $\mathbf{f}(m,e)$ for each candidate.
Third, the LLM induces a decision tree
$\mathcal{T}_{\mathrm{LLM}}$ from the candidate evidence. Each candidate
traverses the tree to obtain a decision path $\mathcal{P}(m,e)$, and the
corresponding adjustments are accumulated with the prior score:
\begin{equation}
\label{eq:algorithm_llm_score}
s_{\mathrm{LLM}}(m,e)
=
s_{\mathrm{prior}}(m,e)
+
\sum_{n\in\mathcal{P}(m,e)}\Delta_n.
\end{equation}
The candidate with the highest $s_{\mathrm{LLM}}(m,e)$ is returned as
the final prediction.

\subsection{Prompt Design}
\label{app:prompt}
This section presents the prompt templates used in the offline semantic augmentation and online evidence-reasoning stages of \ModelName{}.

\noindent \textbf{A. Semantic Augmentation (Offline Stage).}
The offline prompts generate semantic descriptions for KB entities and context-grounded descriptions for mentions. The mention prompt explicitly prohibits predicting canonical entity names or identifiers, ensuring that the generated text is used only for semantic augmentation.

\begin{tcolorbox}[
    enhanced,
    float,
    floatplacement=!t,
    sharp corners,
    boxrule=0.5pt,
    colback=white,
    colframe=black,
    label={box:prompt_design_1},
    title=\textbf{Prompt Design A: Semantic Augmentation},
    fonttitle=\footnotesize\bfseries\sffamily,
    coltitle=black,
    attach boxed title to top left={yshift=-2mm, xshift=2mm},
    boxed title style={colback=white, colframe=white},
    drop shadow
]
\small

\begin{enumerate}[
    label=(\arabic*),
    leftmargin=*,
    topsep=1pt,
    itemsep=0pt,
    parsep=0pt,
    partopsep=0pt
]

  \item \textbf{Task:}
  Generate a concise and informative semantic description of a KB entity
  or a context-grounded mention.

  \item \textbf{Inputs:}
    \begin{itemize}[
        label=$\bullet$,
        leftmargin=*,
        topsep=0pt,
        itemsep=0pt,
        parsep=0pt,
        partopsep=0pt
    ]
      \item Input Type: \texttt{\{entity / mention\}}
      \item Target Name or Mention: \texttt{"\{target\}"}
      \item Context: \texttt{"\{context\}"}
      \item Entity Type: \texttt{\{instance\_type\}} 
      \item Known Attributes: \texttt{"\{attr\_str\}"}
    \end{itemize}

  \item \textbf{Instructions:}
    \begin{enumerate}[
        label=(\alph*),
        leftmargin=*,
        topsep=0pt,
        itemsep=0pt,
        parsep=0pt,
        partopsep=0pt
    ]
      \item Generate a coherent semantic description of 30--50 words
      using the available inputs.
      \item For a KB entity, summarize its name, type, and known attributes.
      \item For a mention, describe its meaning according to the given
      context without predicting a canonical entity name
      or identifier.
      \item Include only information supported by the provided inputs.
    \end{enumerate}

  \item \textbf{Output format:}
  \texttt{[Semantic Description]}

\end{enumerate}

\end{tcolorbox}

\noindent \textbf{B. Evidence Reasoning (Online Stage).}
The online prompt takes the prior score and the 14-dimensional evidence
vector $\mathbf{f}(m,e)$ as input. It induces an executable decision tree
that evaluates multi-perspective support, evidence dominance, and modality availability to rerank candidates.
\begin{tcolorbox}[
    enhanced,
    float,
    floatplacement=!t,
    sharp corners,
    boxrule=0.5pt,
    colback=white,
    colframe=black,
    label={box:prompt_design_2},
    title=\textbf{Prompt Design B: Decision-Tree Induction},
    fonttitle=\footnotesize\bfseries\sffamily,
    coltitle=black,
    attach boxed title to top left={yshift=-2mm, xshift=2mm},
    boxed title style={colback=white, colframe=white},
    drop shadow
]
\small

\begin{enumerate}[
    label=(\arabic*),
    leftmargin=*,
    topsep=1pt,
    itemsep=0pt,
    parsep=0pt,
    partopsep=0pt
]

  \item \textbf{Task:}
  Design an evidence-based decision tree for candidate re-ranking in
  multimodal entity linking.

  \item \textbf{Inputs:}
    \begin{itemize}[
        label=$\bullet$,
        leftmargin=*,
        topsep=0pt,
        itemsep=0pt,
        parsep=0pt,
        partopsep=0pt
    ]
      \item An initial prior score aggregated from normalized
      retrieval-channel scores;
      \item Instance-centric, group-level, and lexical evidence;
      \item Statistical evidence, including the mean, maximum,
      dominance gap, and modality-availability indicators.
    \end{itemize}

  \item \textbf{Guidelines:}
    \begin{itemize}[
        label=$\bullet$,
        leftmargin=*,
        topsep=0pt,
        itemsep=0pt,
        parsep=0pt,
        partopsep=0pt
    ]
      \item Treat the prior score as the initial ranking signal rather
      than the final decision.
      \item Prefer candidates supported by consistent evidence across
      multiple perspectives.
      \item Down-weight high similarity supported by only one evidence
      channel.
      \item Use modality-availability indicators to avoid evaluating
      unavailable evidence channels.
      \item Do not use entity labels, candidate identities, or
      ground-truth information.
      \item Limit the decision tree to a maximum depth of five.
    \end{itemize}

  \item \textbf{Output format:}
  Generate a decision tree in which each node specifies an evidence
  feature, comparison operator, threshold, and score adjustment.
  Output only the executable decision logic without additional
  explanation.

\end{enumerate}

\end{tcolorbox}


\subsection{Entity Sampling for Semantic Augmentation}
\label{app:entity_sampling}

For offline semantic augmentation, we randomly sample 4,000 entities from each dataset using a fixed random seed of 42, forming the subset $E_k \subset E$.

\subsection{Alternative Evidence-Combination Methods}
\label{app:reasoning_alternatives}
The alternatives use the same evidence representation and candidate setting. Mean fusion averages evidence scores, while Reciprocal Rank Fusion (RRF) aggregates channel rankings with larger contributions from higher-ranked candidates. We also compare a conventional decision tree and a single global LLM tree shared across queries. Table~\ref{tab:reasoning_alternatives} shows that none outperforms prior scoring, whereas the query-specific LLM tree improves Hit@1 on all three benchmarks.

\begin{table}[!htbp]
\centering
\resizebox{\linewidth}{!}{%
\begin{tabular}{lccc}
\toprule
Method & WikiMEL & RichpediaMEL & WikiDiverse\\
\midrule
Prior scoring & 81.43 & 72.34 & 66.37\\
Mean fusion & 79.70 & 65.19 & 61.22\\
RRF & 79.97 & 67.83 & 63.22\\
Conventional decision tree & 80.27 & 69.91 & 64.33\\
Global LLM tree & 78.56 & 66.70 & 61.52\\
Query-specific LLM tree & 84.99 & 74.40 & 73.55\\
\bottomrule
\end{tabular}}
\caption{Evidence-combination alternatives (Hit@1, \%).}
\label{tab:reasoning_alternatives}
\end{table}

\begin{table}[!htbp]
\centering
\resizebox{\linewidth}{!}{%
\begin{tabular}{llcccc}
\toprule
Dataset & Method & Hit@1 & Tokens/q & Cost/q & Time/q \\
\midrule
WikiMEL
& Direct LLM & 84.39 & 148.29 & 2.27 & 0.675 \\
& Ours & 84.99 & 67.83 & 1.03 & 0.458 \\
RichpediaMEL
& Direct LLM & 75.78 & 142.07 & 2.16 & 0.677 \\
& Ours & 74.40 & 66.12 & 0.99 & 0.460 \\
WikiDiverse
& Direct LLM & 73.89 & 145.76 & 2.21 & 0.669 \\
& Ours & 73.55 & 66.79 & 1.01 & 0.453 \\
\bottomrule
\end{tabular}}
\caption{Comparison with direct LLM prediction.
Cost/query is reported in units of $10^{-5}$ USD
and time/query in seconds.}
\label{tab:direct_llm}
\end{table}

We further compare the query-specific LLM tree
with direct LLM prediction. As shown in
Table~\ref{tab:direct_llm}, our method achieves
comparable Hit@1 while reducing token usage
and API cost by approximately 54\% and
inference time by approximately 32\%.

\subsection{LLM Reranking Stability}
\label{app:stability}
We repeat LLM-driven reranking for five independent runs, regenerating the query-specific decision trees in each run. Hit@1 is $84.99\pm0.06$, $74.40\pm0.10$, and $73.55\pm0.13$ on WikiMEL, RichpediaMEL, and WikiDiverse, respectively. The induced thresholds and score adjustments vary slightly across generations, while the aggregate performance remains stable. No tree parsing or execution failures are observed in these runs.

\subsection{Simpler Graph Aggregation Alternatives}
\label{app:graph_alternatives}
Mean-neighbor aggregation directly averages neighboring representations. PPR aggregation instead weights neighboring representations according to their PPR relevance, while GCN performs learned graph propagation without teacher--student alignment. Table~\ref{tab:graph_alternatives} compares these alternatives with the full offline evidence-synthesis pipeline.

\begin{table}[!htbp]
\centering
\resizebox{\linewidth}{!}{%
\begin{tabular}{lccc}
\toprule
Method & WikiMEL & RichpediaMEL & WikiDiverse\\
\midrule
Full graph synthesis & 84.99 & 74.40 & 73.55\\
Mean-neighbor aggregation & 64.69 & 59.76 & 62.73\\
PPR aggregation & 75.13 & 63.91 & 63.77\\
GCN & 80.08 & 70.71 & 65.24\\
\bottomrule
\end{tabular}}
\caption{Simpler graph alternatives (Hit@1, \%).}
\label{tab:graph_alternatives}
\end{table}

\subsection{Progressive Evidence Subset Analysis}
\label{app:subsets}
We construct a series of evidence subsets with increasing information. S5 includes four instance-centric similarities and one lexical score. S9 further adds four group-level similarities, S12 incorporates $\mu$, $s_{\max}$, and $\Delta$, and S14 additionally includes two image-availability indicators. As shown in Table~\ref{tab:progressive_subsets}, performance improves as more evidence groups are introduced, supporting their complementary roles. These results are based on selected progressive subsets and do not imply that every individual evidence dimension is nonredundant.

\begin{table}[!htbp]
\centering
\resizebox{\linewidth}{!}{%
\begin{tabular}{lccc}
\toprule
Subset & WikiMEL & RichpediaMEL & WikiDiverse\\
\midrule
S5 & 73.19 & 63.97 & 61.12\\
S9 & 81.87 & 71.90 & 65.73\\
S12 & 82.78 & 73.12 & 66.80\\
S14 & 84.99 & 74.40 & 73.55\\
\bottomrule
\end{tabular}}
\caption{Progressive evidence subsets (Hit@1, \%).}
\label{tab:progressive_subsets}
\end{table}

\subsection{Comparison with Low-label Supervised Trees}
\label{app:supervised_trees}
We compare the unsupervised query-specific LLM tree with supervised decision trees trained using 1\% and 5\% labeled data, keeping the same evidence representation and candidate setting. Table~\ref{tab:supervised_trees} reports their Hit@1 results.

\begin{table}[!htbp]
\centering
\resizebox{\linewidth}{!}{%
\begin{tabular}{lcccc}
\toprule
Method & Labels & WikiMEL & RichpediaMEL & WikiDiverse\\
\midrule
Query-specific LLM tree & 0\% & 84.99 & 74.40 & 73.55\\
Supervised tree & 1\% & 78.87 & 68.12 & 63.44\\
Supervised tree & 5\% & 78.93 & 68.71 & 63.48\\
\bottomrule
\end{tabular}}
\caption{Low-label supervised-tree comparison (Hit@1, \%).}
\label{tab:supervised_trees}
\end{table}

\subsection{Parameter Sensitivity Analysis}
\label{app:sensitivity}
\begin{figure}[t]
    \centering
    \includegraphics[width=1\linewidth]{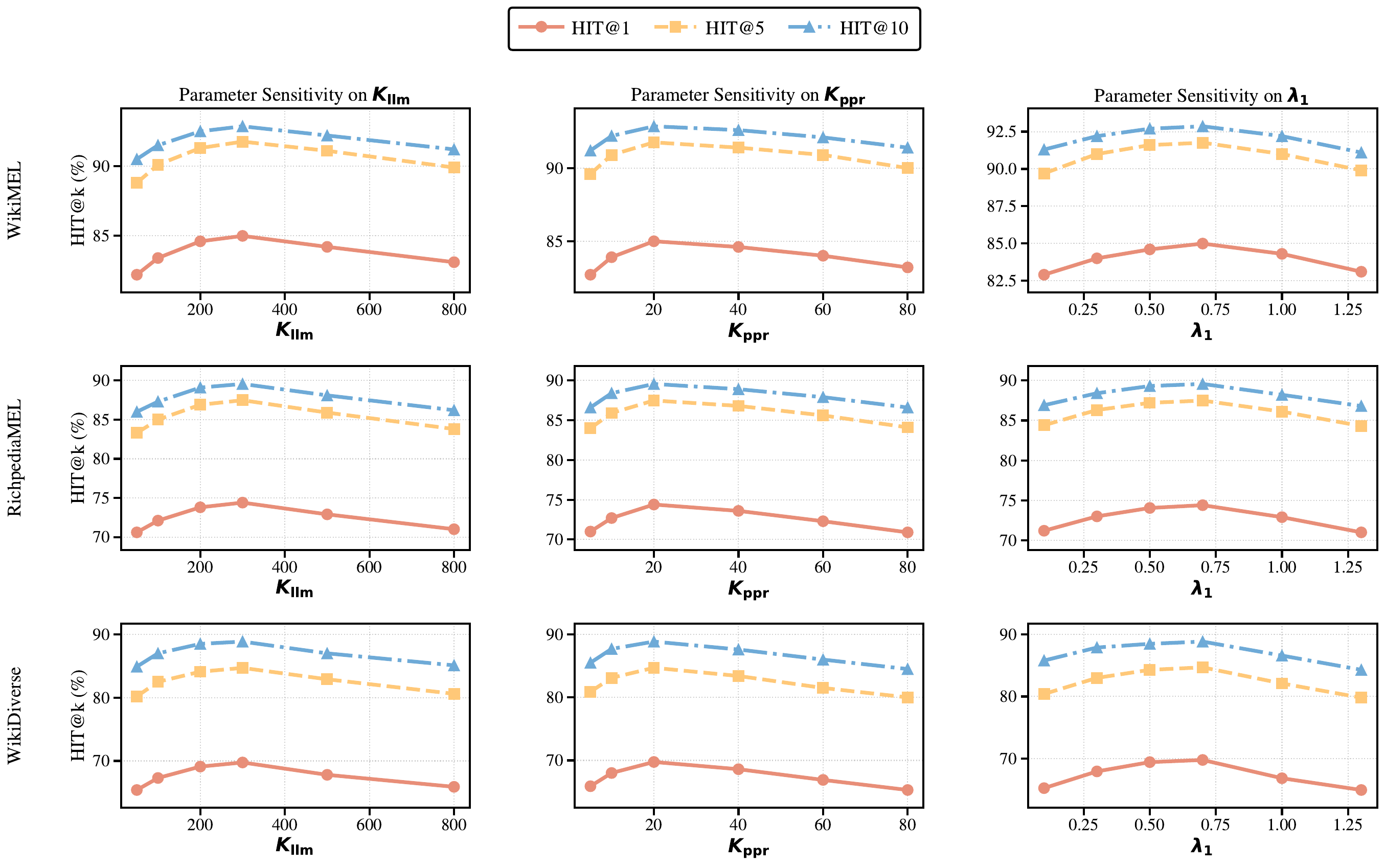}
    \caption{Sensitivity of \ModelName{} to $K_{\mathrm{llm}}$, $K_{\mathrm{ppr}}$, and $\lambda_1$ on three MEL benchmarks.}
    \label{fig:sensitivity_analysis}
\end{figure}
We analyze the sensitivity of \ModelName{} to three offline hyperparameters: the LLM expansion budget $K_{\mathrm{llm}}$, the PPR subgraph size $K_{\mathrm{ppr}}$, and the distillation weight $\lambda_1$.
As shown in Figure~\ref{fig:sensitivity_analysis}, performance generally improves at moderate parameter values and decreases at overly small or large settings. The best results are consistently obtained around $K_{\mathrm{llm}}=300$, $K_{\mathrm{ppr}}=20$, and $\lambda_1=0.75$. Although WikiDiverse exhibits slightly larger variations, \ModelName{} remains stable over reasonable parameter ranges.

\end{document}